\documentclass{article}
\pdfoutput=1

\usepackage{microtype}
\usepackage{graphicx}
\usepackage{subfigure}
\usepackage{booktabs} % for professional tables
\usepackage[numbers]{natbib}

\usepackage[nohyperref,accepted]{icml2019}

\icmltitlerunning{On Random Subsampling of Gaussian Process Regression: A Graphon-Based Analysis}

\usepackage[utf8]{inputenc} % allow utf-8 input
\usepackage[T1]{fontenc}    % use 8-bit T1 fonts
\usepackage{url}            % simple URL typesetting
\usepackage{booktabs}       % professional-quality tables
\usepackage{amsfonts}       % blackboard math symbols
\usepackage{nicefrac}       % compact symbols for 1/2, etc.
\usepackage{microtype}      % microtypography

\usepackage{algorithm}
\usepackage{algorithmicx}
\usepackage[noend]{algpseudocode}
\usepackage{amsthm,amsmath,amssymb,bm,mathrsfs}
\usepackage{color,comment}
\usepackage{dutchcal}
\usepackage{xspace}
\usepackage{graphicx}
\usepackage{multicol}
\usepackage{here}

\newtheorem{theorem}{Theorem}[section]
\newtheorem{lemma}[theorem]{Lemma}
\newtheorem{definition}[theorem]{Definition}

\newtheorem{corollary}[theorem]{Corollary}

\newcommand{\bbN}{\mathbb{N}}
\newcommand{\bbR}{\mathbb{R}}
\newcommand{\bmf}{{\bm f}}
\newcommand{\bmk}{{\bm k}}
\newcommand{\bmu}{{\bm u}}
\newcommand{\bmv}{{\bm v}}
\newcommand{\bmw}{{\bm w}}
\newcommand{\bmx}{{\bm x}}
\newcommand{\bmy}{{\bm y}}
\newcommand{\bmz}{{\bm z}}
\newcommand{\bmxi}{{\bm \xi}}
\newcommand{\bmDelta}{{\bm \Delta}}
\newcommand{\bmone}{{\bm 1}}
\newcommand{\bmalpha}{{\bm \alpha}}

\newcommand{\calv}{\mathcal{v}}
\newcommand{\calk}{\mathcal{k}}

\newcommand{\caly}{\mathcal{y}}

\newcommand{\calA}{\mathcal{A}}
\newcommand{\calK}{\mathcal{K}}
\newcommand{\calH}{\mathcal{H}}
\newcommand{\rmd}{\mathrm{d}}
\newcommand{\set}[1]{\{#1\}}
\newcommand{\Ep}{\mathbb{E}}
\newcommand{\argmin}{\mathop{\mathrm{argmin}}}
\renewcommand{\hat}{\widehat}
\renewcommand{\tilde}{\widetilde}

\newcommand{\nystrom}{Nystr\"{o}m\xspace}

\begin{document}

\twocolumn[
%\icmltitle{Instant Learning and Inference of Gaussian Process Regression}
\icmltitle{On Random Subsampling of Gaussian Process Regression:\\ A Graphon-Based Analysis}

\begin{icmlauthorlist}
\icmlauthor{Kohei Hayashi (Preferred Networks)}{}
\icmlauthor{Masaaki Imaizumi (Institute of Statistical Mathematics)}{}
\icmlauthor{Yuichi Yoshida (National Institute of Informatics)}{}
\end{icmlauthorlist}
%\icmlaffiliation{pfn}{Preferred Networks, Japan}
%\icmlaffiliation{ism}{Institute of Statistical Mathematics, Japan}
%\icmlaffiliation{nii}{National Institute of Informatics, Japan}

\vskip 0.3in
]

\begin{abstract}
%   The accuracy-runtime trade-off is one of the main concerns of approximating Gaussian process regression.
%   Although there are a lot of studies for the approximation that have theoretical guarantees of accuracy such as the \nystrom method and random Fourier expansion, their time complexities still have a large dependency on the number of samples, which can be prohibitive in the case of large amounts of training data.
%   In this paper, we explore the extreme point of the trade-off --- constant-time algorithms for learning Gaussian process regression, which are equivalent to subsampling a constant number of samples from the training data.
%   Using the machinery of graphons, we analyze the generalization ability in addition to the accuracy of the predictive mean and variance.
%   We also investigate the performance of cross-validation with a constant number of subsamples. Numerical analysis shows the usefulness of the algorithm.

In this paper, we study random subsampling of Gaussian process regression, one of the simplest approximation baselines, from a theoretical perspective. Although subsampling discards a large part of training data, we show provable guarantees on the accuracy of the predictive mean/variance and its generalization ability.
For analysis, we consider embedding kernel matrices into graphons, which encapsulate the difference of the sample size and enables us to evaluate the approximation and generalization errors in a unified manner. The experimental results show that the subsampling approximation achieves a better trade-off regarding accuracy and runtime than the \nystrom and random Fourier expansion methods.
\end{abstract}

\section{Introduction}

Gaussian process regression (GPR) is a fundamental tool for supervised learning. After learning parameters, we can make predictions in a distributional form, which is useful for measuring the uncertainty of the predictions. Of course, to enjoy such flexibility, we need to pay the price --- computationally. For the number of samples $n$, both training (parameter learning) and the computation of the predictive distributions require polynomial time in $n$. The dominant part is the computation of the inverse of the $n$-by-$n$ kernel matrix, which requires $O(n^3)$ time.

To reduce the time complexity, a lot of sophisticated approximation methods have been developed. Most of them introduce some structure into the kernel matrix to approximate it. For example, the \nystrom method~\cite{williams2001using} approximates the kernel matrix with a low-rank matrix. Given a shift-invariant kernel function, the random Fourier expansion (RFE)~\cite{rahimi2008random} approximately constructs a feature function in a finite-dimensional space. Several methods exploit specific properties of kernel matrices~\cite{Pleiss:2018vp,wilson2015kernel}.

A more drastic approach is \emph{subsampling}, i.e., training GPR with a subset of the data. If we pick subsamples completely randomly, the time complexity only depends on the subsample size $s$, which is independent of $n$. While its simplicity and the computational cheapness, random subsampling has been seen as a baseline rather than a competitive method in the GP community~\cite{rasmussen2004gaussian,snelson2007flexible,quinonero2005unifying}. One of the main reasons is that it completely discards a large part of the data, and it seems impossible to estimate the uncertainties~\cite{quinonero2005unifying}. Also, its theoretical justification is non-trivial, because subsampling changes the size of the kernel matrix. This is contrastive to the case of the structure-based approximations, which retain the size of the kernel matrix as $n$-by-$n$ and can directly evaluate its approximation accuracy as the prediction accuracy, whereas they require at least $\Omega(n)$ computational cost in training.

In this paper, we study the subsampling approximation of GPR from a theoretical perspective. Somewhat unexpectedly, our main results show that subsampling can maintain global information with a sufficiently small number of subsamples. More specifically, with any bounded data and kernel functions, $\Theta(s)$ subsamples guarantee $O(\log^{-1/4} s)$ prediction error at any new data point.

For analysis, we exploit the machinery of \emph{graphons}, a continuous limit of bounded symmetric matrices, which have effectively been used in graph theory (see~\cite{Lovasz:2012wn}).
Embedding the kernel matrices into graphons abstracts their difference in terms of the (sub)sample size, which enables to evaluate the predictive mean/variance without using strong statistical assumptions (Theorem~\ref{the:prediction-linfty}). Because graphons can handle infinitely large matrices, i.e., the kernel matrices with $n\to\infty$, the result is immediately applicable to evaluate the generalization error (Corollary~\ref{cor:gen}).
Moreover, we show that, with a constant number of subsamples, hyperparameter tuning based on cross-validation (CV) succeeds with a high probability (Theorem~\ref{thm:cv2}).
We performed experiments that provided encouraging results of subsampling in terms of the speed-accuracy trade-off.

%   \begin{table}[H]
%     \caption{Comparison of time complexity.}\label{tab:comparison}
%     \centering
%     \begin{tabular}{rccccc}
%   \hline
%  Method & Training & Pred. mean & Pred. var. & CV & Provable guarantees \\
%   \hline
%  Original  & $O(n^3)$ & $O(n)$ & $O(n^2)$ & $O(n^3)$ & N/A \\
%  \nystrom  & $O(n)$ & $O(n)?$ & $O(n)?$ & $O(n)$ & \checkmark\\
%  KISS      & $O(n)$ & $O(1)$ & $O(1)$ & $O(n)$ &  \\
%  Proposed  & $O(1)$ & $O(1)$ & $O(1)$ & $O(1)$ & \checkmark \\
%    \hline
% \end{tabular}

%   \end{table}

%  \ynote{Give some comparison with subsampling (like to get theoretical guarantee with subsampling, we need many assumptions like differentiability of $f^*$.)}

%!TEX root=./main.tex

\section{Preliminaries}
For an integer $n \in \bbN$, we denote the set $\set{1,2,\ldots,n}$ by $[n]$.
For $a,b\in \bbR$ and $c \in \bbR_+$, we mean $b -c \leq a \leq b + c$ by $a = b \pm c$.

For vectors $\bmx,\bmy\in \bbR^p$, $\langle \bmx,\bmy\rangle$ denotes their inner product.
For a vector $\bmx\in \bbR^d$ and a set $S \subseteq [n]$, $\bmx_S \in \bbR^{|S|}$ denotes the vector obtained by restricting $\bmx$ to $S$.
Similarly, for a matrix $A \in \bbR^{n \times m}$ and sets $S \subseteq [n]$ and $T \subseteq [m]$, $A_{ST} \in \bbR^{|S| \times |T|}$ denotes the matrix obtained by restricting $A$ to $S \times T$.
For a matrix $A \in \bbR^{n \times m}$, we define $\|A\|_{\max}$ as $\max_{i\in [n],j\in [m]}|A_{ij}|$.
$\mathscr{N}(\mu,\sigma^2)$ denotes the Gaussian distribution with mean $\mu$ and variance $\sigma^2$.

\subsection{Gaussian Process Regression}
Let $(\bmx_1,y_1),\ldots,(\bmx_n,y_n) \in \bbR^p \times \bbR$ be training samples.
The goal of the GPR is to obtain a predictive distribution for $f^*(\bmx^*)$ when a new sample $\bmx^* \in \bbR^p$ arrives.
In this work, we consider the zero-mean GP prior for $f$ with the covariance kernel function $k\colon \bbR^p \times \bbR^p \to \bbR$.
When the variance of the observation noise is specified as $\nu^2>0$, the predictive distribution for $f^*(\bmx^*)$ is given as the following Gaussian distribution:
\begin{align}
	&\mathscr{N}\Bigl(\bmk^T{(K + \nu^2 I)}^{-1} \bmy, k(\bmx^*,\bmx^*) - \bmk^T{(K + \nu^2 I)}^{-1} \bmk\Bigr) \label{def:pred}
\end{align}
where $K \in \bbR^{n \times n}$ is the kernel matrix with $K_{ij} = k(\bmx_i,\bmx_j)\;(i,j\in [n])$ and $\bmk = {(k(\bmx^*,\bmx_i))}_{i \in [n]} \in \bbR^n$ (see Section~2 of~\cite{rasmussen2004gaussian} for more details).

Let $\calH$ be the reproducing kernel Hilbert space (RKHS) associated with $k(\cdot, \cdot)$.
For a vector $\bmx \in \bbR^p$, let $\phi_{\bmx} = k(\bmx, \cdot) \in \calH$ be the element corresponding to $\bmx$.
Note that $k(\bmx_i,\bmx_j) = \langle \phi_{\bmx_i},\phi_{\bmx_j}\rangle_{\calH}$.
We define a linear operator $\Phi\colon \bbR^n \to \calH$ as $\Phi(\bmw) = \sum_{i \in [n]}\phi_{\bmx_i}w_i$.

\subsection{Graphons and Matrices}\label{sec:graphon}
A (measurable) bounded symmetric function $f\colon {[0, 1]}^2 \to \bbR$ is called a \emph{graphon}\footnote{Precisely speaking, such a function is called a \emph{kernel} and a (measurable) symmetric function $f\colon {[0, 1]}^2 \to [0,1]$ is called a \emph{graphon} in the literature. However, to avoid confusion with the kernel function $k(\cdot,\cdot)$, we adopt the term graphon here.}.
We can regard a graphon as a matrix in which the index is specified by a real value in $[0, 1]$.
For two functions $f, g \colon [0, 1] \to \bbR$, we define their \emph{inner product} as $\langle f,g\rangle  = \int_0^1 f(x)g(x)dx$.
We also define their \emph{outer product} $f g^\top\colon {[0,1]}^2 \to \bbR$ as $fg^\top(x,y)=f(x)g(y)$.
For a graphon $\calA\colon {[0,1]}^2 \to \bbR$ and a function $f \colon [0,1] \to \bbR$, we define the function $\calA f \colon [0, 1] \to R$ as $(\calA f )(x) = \langle \calA(x, \cdot), f \rangle$.

For an integer $n \in \bbN$, let $I_1^n =[0,\frac{1}{n}]$, and for every $1<k\leq n$, let $I_k^n =(\frac{k-1}{n},\frac{k}{n}]$. For $x\in[0,1]$, we define $i_n(x)$ as the unique integer $k \in [n]$ such that $x \in I_k^n$.

\begin{definition}
  Given a vector $\bmv \in \bbR^n$, we construct the corresponding function $\calv\colon [0,1]\to \bbR$ as $\calv(x) = v_{i_n(x)}$.
  In addition, given a set of indices $S \subseteq [n]$, when we write $\calv_{S}$, we first extract the vector $\bmv_{S} \in \bbR^{|S|}$ and then consider its corresponding function.
  Similarly, given a matrix $A \in \bbR^{n \times n}$, we construct the corresponding graphon $\calA\colon{[0,1]}^2 \to \bbR$ as $\calA(x, y) = A_{i_n(x)i_n(y)}$.
  In addition, given two sets of indices $S \subseteq [n]$ and $T \subseteq [n]$, when we write $\calA_{ST}$, we first extract the matrix $A_{ST}\in \bbR^{|S|\times |T|}$ and then consider its corresponding graphon.
\end{definition}

For a graphon $\calA\colon {[0,1]}^2 \to \bbR$, its \emph{cut norm} is defined as
\[
  \|\calA\|_\square = \max_{S,T \subseteq [0,1]}\Bigl|\int_S \int_T \calA(x,y) \rmd x \rmd y\Bigr|,
\]
where $S$ and $T$ run over all the measurable sets.

The following lemma states that we can approximate a matrix with its small submatrix with respect to the cut norm of the difference of their corresponding graphons.
\begin{lemma}[\cite{Hayashi:2016wh}]\label{lem:Hayashi}
  Let $L>0$ and let $A^1,\ldots,A^T \in {[-L,L]}^{n \times n}$ be matrices.
  Let $S \subseteq [n]$ be a set of $s$ elements that are uniformly selected at random.
  Then, with a probability of at least $1 - \exp(-\Omega(s T/ \log_2s))$,
  there exists a measure-preserving bijection $\pi\colon[0,1] \to [0,1]$ such that, for every $t \in [T]$, we have
  \[
    \|\calA^t - \pi(\calA^t_{SS})\|_{\square} = O\Bigl(L\sqrt{T/\log_2 s}\Bigr).
  \]
  Moreover, we can assume $i_n(\pi(x))=i_n(\pi(y))$ whenever $i_n(x)=i_n(y)$, that is, $\pi$ is a block-wise bijection.
\end{lemma}

The following lemma states that the quadratic form of a graphon with a small cut norm is small.
\begin{lemma}[\cite{Hayashi:2016wh}]\label{lem:Hayashi-approximation}
  Let $\epsilon > 0$ and $\calA\colon {[0,1]}^2 \to \bbR$ be a graphon with $\|\calA\|_\square\leq \epsilon$.
  Then, for any functions $f,g\colon [0,1] \to [-L,L]$, we have $|\langle f, \calA g\rangle| \leq \epsilon L^2$.
\end{lemma}

%!TEX root=./main.tex

\section{Gaussian Process Regression with Graphons}\label{sec:algo}

The main purpose of GPR is to predict a function value at a new data point. The standard statistical result shows that, in a point-wise sense, the predictive mean converges to the true function as the sample size increases under some regularity conditions. In other words, the true function can be rephrased as the limit of the predictive mean of the GPR with infinitely many samples. The prediction accuracy (i.e., the generalization error) is therefore measured by the distance between the finite- and infinite-sample GPRs. However, analyzing the infinite-sample GPR is not trivial because we cannot write down the solution using standard matrix operations such as matrix inverse because the kernel matrix is infinitely large.

Graphons are a generic tool to handle both finite- and infinite-size matrices. First, a kernel matrix with infinitely many samples is embedded into a graphon by taking a map from the sample indices $[n]$ to $[0,1]$. We can then reformulate the predictive distribution as the minimization problem of the quadratic objective function (i.e., the Gaussian log-likelihood of~\eqref{def:pred}) associated with the graphon. Also, a kernel matrix with a finite sample size is embedded into a graphon using the partition $I^n_1,\dots,I^n_n$ defined in Section~\ref{sec:graphon}, which can be seen as the low-resolution version of the infinite one. Now, we can bound the difference between the finite- and infinite-sample objective values by using the distance between the two graphons in terms of the cut norm (using Lemma~\ref{lem:Hayashi-approximation}). The predictive accuracy is also derived in the same manner. We remark that the above approach can be used to analyze the difference between GPRs with different (finite) sample sizes, from which we can derive the accuracy of subsampling.

Using graphons and RKHSs to kernel methods have similar spirits: The kernel trick based on RKHSs provides an explicit form of the regression function when using the infinite-dimensional feature space whereas graphons provide an explicit form of that when using infinitely many samples.

%In this section, we describe a constant-time algorithm for approximating the predictive mean and variance in~\eqref{def:pred} based on subsampling, and discuss its accuracy.

\subsection{Subsampled Predictive Distribution}
First, we rephrase the predictive mean and variance.
For a parameter $\lambda > 0$, we define a normalized loss function.
\begin{align}
    \ell_{K,\bmk,\lambda}(\bmv) = \frac{1}{n}\Bigl\| \frac{K \bmv}{n} - \bmk \Bigr\|_2^2 + \frac{\lambda}{n^2} \langle \bmv, K \bmv\rangle \label{eq:normalized-loss}
\end{align}
By setting $\lambda = \nu^2/n$ and with the solution
\begin{align}
\bmv^*
= \argmin_{\bmv \in \bbR^n} \ell_{K,\bmk,\lambda}(\bmv)
= n{(K + n\lambda I)}^{-1}\bmk, \label{eq:the-problem}
\end{align}
the predictive mean and variance in~\eqref{def:pred} can be rewritten as
\begin{align}
  \mu_{\bmx^*} & := \frac{\langle \bmv^*, \bmy\rangle}{n} ~~ \text{and} ~~
  \sigma_{\bmx^*}^2 := k(\bmx^*,\bmx^*) - \frac{\langle \bmv^*, \bmk\rangle}{n}.
  \label{eq:mean-and-variance}
\end{align}
In what follows, we leave $\lambda$ as a parameter as we often do not know the value of $\nu$.
%Note that the function~\eqref{eq:normalized-loss} is normalized such that its value is $O(1)$ when $\|K\|_{\max}=O(1)$, $\|\bmk\|_\infty=O(1)$, and  $\|\bmv\|_\infty =O(1)$.

Our algorithm consists of two parts.
The first part of our algorithm (Algorithm~\ref{alg:loss}) approximately minimizes~\eqref{eq:normalized-loss}.
For a small integer $s \in \bbN$, it samples a set $S \subseteq [n]$ of size $s$ uniformly at random and then minimizes the function obtained by restricting~\eqref{eq:normalized-loss} to $S$, that is, $\ell_{K_{SS},\bmk_S,\lambda}$.
Here, we assume that the matrix $K$ and vector $\bmk$ are given through query accesses.
That is, if we specify the indices $i,j \in [n] $, we can obtain $K_{ij}$ in constant time, and similarly, if we specify an index $i \in [n] $, we can obtain $k_i$ in constant time.
The second part of our algorithm (Algorithm~\ref{alg:prediction}) computes approximations for $\mu_{\bmx^*}$ and $\sigma^2_{\bmx^*}$ using the vector obtained in the first part.
%For a small integer $s \in \bbN$, we sample a set $S \subseteq [n]$ of size $s$ uniformly at random, and compute the predictive mean and variance by using
%We note that our method is very similar to the subsampled version of GP.
%\ynote{what's the difference?}

\begin{algorithm}[t!]
  \caption{Approximate solver for the normalized loss}\label{alg:loss}
  \begin{algorithmic}[1]
    \Require{$n,s \in \bbN$, $\lambda > 0$, and query accesses to $K \in \bbR^{n \times n}$ and $\bmk \in \bbR^n$.}
    \State{Sample a set $S \subseteq [n]$ of size $s$ chosen uniformly at random.}
    \State{$\tilde{\bmv}^* \leftarrow \argmin_{\tilde{\bmv}} \ell_{K_{SS},\bmk_S,\lambda}(\tilde{\bmv})$.}
    \State{\Return $\tilde{\bmv}^*$ and $S$.}
  \end{algorithmic}
\end{algorithm}

\begin{algorithm}[t]
  \caption{Approximation algorithm for predictive mean and variance}\label{alg:prediction}
  \begin{algorithmic}[1]
    \Require{$n,s \in \bbN$, $\lambda > 0$, and query accesses to $K \in \bbR^{n \times n}$, $\bmk \in \bbR^n$, and $\bmy \in \bbR^n$.}
    \State{Run Algorithm~\ref{alg:loss} to obtain $\tilde{\bmv} \in \bbR^s$ and a subset $S \subseteq [n]$ of size $s$.}
    \State{$\tilde{\mu}_{\bmx^*} \leftarrow \langle \tilde{\bmv},\bmy_S\rangle /s$.}
    \State{$\tilde{\sigma}_{\bmx^*}^2 \leftarrow k(\bmx^*,\bmx^*) - \langle \tilde{\bmv},\bmk_S\rangle /s$.}
    \State{\Return $(\tilde{\mu}_{\bmx^*}, \tilde{\sigma}_{\bmx^*}^2)$.}
  \end{algorithmic}
\end{algorithm}

For the first part of our algorithm, we show the following guarantee, which states that the minima of $\ell_{K,\bmk,\lambda}$ and $\ell_{K_{SS},\bmk_S,\lambda}$ are close. The proof for this is presented in Section~\ref{sec:loss}.
\begin{theorem}\label{the:loss-linfty}
  For any $\epsilon > 0$, Algorithm~\ref{alg:loss} with $s = 2^{\Theta(1/\epsilon^2)}$ outputs $\tilde{\bmv}^* \in \bbR^s$ such that
  \begin{align*}
    \ell_{K_{SS},\bmk_S,\lambda}(\tilde{\bmv}^*)
    = \ell_{K,\bmk,\lambda}(\bmv^*) \pm O\Bigl(\epsilon L^2 R^2 \Bigr)
  \end{align*}
  with a probability of at least $0.99$, where $\bmv^* = \argmin_{\bmv} \ell_{K,\bmk,\lambda}(\bmv)$, $L = \max\set{\|K\|_{\max},\|\bmk\|_\infty}$, and $R = \max\set{\|\bmv^*\|_\infty,\|\tilde{\bmv}^*\|_\infty}$.
  %The running time is $2^{O(1/\epsilon^2)}$.
\end{theorem}

For the second part of our algorithm, we show the following guarantee, which states that the approximations computed using Algorithm~\ref{alg:prediction} are accurate.
The proof for this is presented in Section~\ref{sec:prediction}.
\begin{theorem}\label{the:prediction-linfty}
  Let $\|\bmz\|_\phi:=\inf\{\|\bmalpha\|_\calH : \bmalpha\in\calH, \forall_i z_i=\langle \phi_{\bmx_i},\bmalpha\rangle_\calH\}$ be the norm of $\bmz\in\mathbb{R}^n$ in the feature space spanned by $\{\phi_{\bmx_i}\}_{i\in[n]}$.
  For any $\epsilon > 0$, Algorithm~\ref{alg:prediction} with $s = 2^{\Theta(1/\epsilon^2)}$ and $\lambda=\Theta(1)$ outputs $(\tilde{\mu}_{\bmx^*},\tilde{\sigma}_{\bmx^*}^2)$ such that
  \[
    |\mu_{\bmx^*} - \tilde{\mu}_{\bmx^*}| = O\left( \sqrt{\epsilon} L^2R \right)
    \text{ and }
    |\sigma_{\bmx^*}^2 - \tilde{\sigma}_{\bmx^*}^2| = O\left(\sqrt{\epsilon} L^2R \right),
  \]
  with probability of at least $0.99$,
  where $L = \max\set{\|K\|_{\max},\|\bmk\|_\infty,k(\bmx^*, \bmx^*), \|\bmy\|_\phi}$ and $R = \max\set{\|\bmv^*\|_\infty,\|\tilde{\bmv}^*\|_\infty}$.
  %The running time is $2^{O(1/\epsilon^2)}$.
\end{theorem}

We expect that the above error rates are independent of $n$, i.e., $L=O(1)$ and $R=O(1)$.
For $L$, the condition $\|\bmy\|_\phi=O(1)$ is typically admissible in the noiseless case, which is the scenario we want to approximate the outputs of the exact GPR by subsampling.
For $R$, the fluctuation of $\bmv^*$ and $\tilde{\bmv}^*$ should be tamed by the $L_2$ regularization in which the regularization strength is $n\lambda$ for $\bmv^*$ and $s\lambda$ for $\tilde{\bmv}^*$ (see \eqref{eq:the-problem}).

\subsection{Generalization Error}

We provide generalization analysis for the subsampling method, namely, we investigate how our method estimates an unknown data generating process.
To this end, let us assume that the samples are generated through a function $f^*\colon \bbR^p \to \bbR$ that relates $y_i$ and $\bmx_i$ as
\begin{align}
	y_i = f^*(\bmx_i) + \xi_i, ~ i \in [n], \label{model:reg}
\end{align}
where $\xi_i\sim \mathscr{N}(0,\nu^2)$ is the Gaussian noise.

We note that Theorem~\ref{the:prediction-linfty} holds for any sample size $n$, even at the limit $n \to \infty$.
It is well known that universal kernel functions (e.g., the Gaussian kernel and the polynomial kernel) can approximate any continuous functions~\cite{micchelli2006universal,steinwart2008support}, and several kernel-based estimators converge to any truth functions $f^*$ at $n\to\infty$~\cite{gyorfi2006distribution,rasmussen2004gaussian}.
The result also holds with the GP regression estimator~\cite{van2008rates} with some assumptions.
Along with these results, Theorem~\ref{the:prediction-linfty} can be used to bound the generalization error.
\begin{corollary}\label{cor:gen}
    Consider the same setting as in Theorem~\ref{the:prediction-linfty} and assume that the observations follow the model \eqref{model:reg}.
    Suppose $\mu_{\bmx^*}$ is a consistent estimator for $f^*(\bmx^*)$, namely, $|\mu_{\bmx^*} - f^*(\bmx^*)|\to 0$ as $n \to \infty$.
    Then, with probability at least 0.98, the following holds:
    \begin{align*}
        &|\tilde{\mu}_{\bmx^*} - f^*(\bmx^*)|  = O\biggl(\frac{L^{'2} R}{\log^{1/4} n} \biggr), % + O_P(\eta_n).
    \end{align*}
    where $L' = \max\set{\|K\|_{\max},\|\bmk\|_\infty,k(\bmx^*, \bmx^*), \|f^*\|_\calH}$.
    Furthermore, if $k(\bmx,\bmx)$ and $\mu_{\bmx}$ are bounded for all $\bmx$, then, with probability at least 0.98, the following holds:
    \begin{align*}
        &\|\tilde{\mu}_{\cdot} - f^*\|_{L^2}
        = O\biggl(\frac{L^{'2} R}{\log^{1/4} n} \biggr), % + O_P(\eta_n).
    \end{align*}
    where $\|\cdot\|_{L^2}$ denotes the $L^2$-norm for square integrable functions.
\end{corollary}

Although Corollary~\ref{cor:gen} only guarantees a relatively slow rate of $O(\log^{-1/4} n)$, besides the consistency assumption on $\mu_{\bmx^*}$, it does not require any other assumption such as the differentiability of $f^*$.
%, it provides a generalization bound for a considerably general setting.
%which works with more general settings.

% In Section~\ref{sec:cv}, we discuss applications of our method to cross-validation.

\section{Related Work}\label{subsec:comparison}

\if0

\begin{table*}[htbp]
     \centering
     \begin{tabular}{cccccc}
   \hline
  Method & Comp. Complex. & Error Bound & Variance Error Bound & Assumption\\
   \hline
%  Original GPR & $O(n^3)$ & $0$ & $\surd$ & $\surd$ & Smoothness \\
  \nystrom  & $O(ns^2 + s^3)$ & $\rm{poly}(s)$ &  N/A &  ? \\
  SR & $O(ns^2)$ & N/A & N/A & ?\\
  LS & $O(ns^2 + s^3)$ & $\rm{poly}(s)$ & N/A &   Incoherence\\
  RLS & $O(ns^2)$ & $\rm{poly}(s)$ & N/A   & None\\
  Column Sampling & $O(ns^2)$ & $\rm{poly}(s)$ &  N/A  & Incoherence\\
  Sketching & $O(ns^2)$ & $\rm{poly}(s)$ & N/A   & Incoherence\\
  RFE & $O(ns^2)$ & $\rm{poly}(s)$ & N/A &  Shift invariant\\
   KISS\cite{Pleiss:2018vp}& $O(ns)$ & N/A & N/A  &  ? \\
  \textbf{Subsampling} (Ours)  & $O(s^3)$ & $O(1/\sqrt{\log s})$ & $O(1/\sqrt{\log s})$ & None\\
    \hline
 \end{tabular}
     \caption{Comparison of the methods for GPR.}\label{tab:comparison}
\end{table*}
 \fi

\begin{table*}[htbp]
  \centering
  \caption{Comparison of approximation methods for GPR.}\label{tab:comparison}
  \begin{tabular}{cccccc}\toprule
  Method & Time Complexity & Predictive Mean Error & Predictive Variance Error & Assumptions\\
   \midrule
  \nystrom  & $O(ns^2)$ or $O(ns^2 + s^3)$ & $\tilde{O}(s^{-\gamma})$ &  N/A & Incoherence \\
  RFE & $O(ns^2)$ & $\tilde{O}(s^{-1/2})$ &  N/A & Restriction on kernels \\
  Lanczos& $O(n+s)$ & N/A & N/A  &  None \\
  \textbf{Subsampling} & $O(s^3)$ & $O(\log^{-1/4} s)$ & $O(\log^{-1/4} s)$ & None\\
    \bottomrule
 \end{tabular}
\end{table*}

\subsection{Approximation Analysis}

Subsampling-based approximations are known as the subset of the data (SD) methods, which has several variants in terms of how the subsamples are chosen~\cite{quinonero2005unifying}. The simplest version chooses samples completely randomly, which is equivalent to our algorithms except that the simplest SD method fixes the noise variance $\nu^2$, independently of the subsample size $s$, whereas ours scales $\nu^2$ to derive a theoretical guarantee on its accuracy.
Other SD methods select subsamples based on more sophisticated criteria such as the differential entropy score~\cite{herbrich2003fast}, which requires, however, $O(n)$ time as it scans all the samples.

The inducing points methods~\cite{quinonero2005unifying,snelson2006sparse,titsias2009variational} are another class of approximation methods, which picks up a small number of auxiliary variables as pseudo-samples---inducing points---and approximate the predictive mean using the cross-covariance between the inducing points and the rest of the samples.
The inducing points are usually chosen based on the marginal likelihood~\cite{snelson2006sparse} or the variational principle~\cite{titsias2009variational}.
Although they perform well in practice~\cite{matthews2017scalable}, their time complexity depends on $n$ due to computing the cross-covariance.
Also, to the best of our knowledge, their theoretical properties, especially the approximation accuracy, have not been studied.

The \nystrom method and its variants such as the leverage score method are also intensively studied~\cite{alaoui2015fast,bauer2016understanding,gittens2016revisiting,musco2017recursive,williams2001using}.
They employ $s < n$ points as regressors and their time complexity is typically $O(ns^2)$.
Assuming that the selected regressors have a nice property such as incoherence, their approximation error for the predictive mean is $\tilde{O}(s^{-\gamma})$, which follows from the approximation guarantee for the kernel matrix in the spectral norm~\cite{musco2017recursive}.
Here, $\gamma > 0$ is a parameter depending on the kernel function.

RFE approximates predictors by using $s$ Fourier bases~\cite{avron2017random,sriperumbudur2015optimal,yang2012nystrom}, which requires $O(ns^2)$ time and some restriction on kernel functions such as shift-invariance.
The approximation error for the predictive mean is $\tilde{O}(s^{-1/2})$, which follows from the error analysis for the kernel matrix~\cite{yang2012nystrom}.
Also, some other works~\cite{yang2012nystrom,sriperumbudur2015optimal} analyzed its generalization capability.

Pleiss~et~al.~\cite{Pleiss:2018vp} developed Lanczos approximation.
The time complexity is $O(n+s)$, where $s$ is the number of inducing points.
No theoretical guarantee is known.

\textbf{Time complexity:}
Note that subsampling requires only $O(s^3)$ time, which is $O(1)$  when $s=O(1)$.
In contrast, all the other methods depend on $n$, and hence they cannot be run in $O(1)$ time.

\textbf{Error bound:}
As for the error bound, recalling the relation $s = 2^{\Theta(1/\epsilon^2)}$ in Theorem~\ref{the:prediction-linfty}, subsampling has a convergence rate of $O(\log^{-1/4} s)$, which is slower than the polynomial rates achieved by the \nystrom method and RFE\@.
However, we stress here that, at the cost of the slow convergence rate, we eliminated several assumptions used in their analysis.
More specifically, the \nystrom method requires that the subsampled regressors are incoherent and the RFE require that the kernel function is shift-invariant.
Also, we can provide an error bound for the predictive variance, which has not been addressed in the \nystrom, RFE, or Lanczos methods.

Table~\ref{tab:comparison} summarizes our theoretical results for the subsampling method against those for other approximation methods.

\subsection{Generalization Analysis}

Some existing studies have developed generalization theory of GPR\@.
Van~der~Vaart~et~al.~\cite{van2008rates,vaart2011information} evaluated GPR by using the notion of \textit{posterior contraction}, and showed that the generalization error measured by the $L^2$-norm is
\begin{align}
    O \left( n^{-2\beta / (2\beta + p)} \right), \label{ineq:gen_related}
\end{align}
where $\beta$ is the number of differentiability of $f^*$.
%It is known that the rate is a minimax optimal rate for learning smooth functions \cite{gine2016mathematical}.
For different metrics such as the $L^\infty$-norm, the same rates (up to logarithmic factors) were obtained~\cite{gine2011rates,yoo2016supremum}.
%Furthermore, it is shown that GPR with the \nystrom method can achieve same convergence rate \cite{rudi2015less}.

Our generalization analysis (Corollary~\ref{cor:gen}) provides the rate of $O( \log^{-1/4}n )$, which is much slower than~\eqref{ineq:gen_related}.
Nevertheless, this rate only requires the consistency of $\mu_{\bmx^*}$ and does not impose any assumption on $f^*$, while the existing rate~\eqref{ineq:gen_related} assumes the differentiability of $f^*$.

%Our error bounds (Theorem~\ref{the:prediction-linfty}) have several advantages over the existing bounds for the subsampling method and other approximation methods such as the \nystrom method~\cite{drineas2005nystrom,gittens2016revisiting,smola2000sparse}.
%For our discussion, let $\tilde{\mu}^{N}_{\bmx^*}$ be the output of the subsampling approximation for $\bmx^*$.

%Firstly, our bounds are pointwise, i.e., we can bound $|\mu_{\bmx^*} - \tilde{\mu}_{\bmx^*}|$ for arbitrary $\bmx^*$.
% instead of the integral $\int (\mu_{\bmx^*} - \tilde{\mu}_{\bmx^*})^2 dP(\bmx^*)$.
% Hence, we can evaluate a size of errors for each prediction.
%In contrast, all the existing bounds~\cite{drineas2005nystrom,gittens2016revisiting,rudi2015less}
%\ynote{Why \cite{gittens2016revisiting} is missing?}
%considered the integrated error with respect to $\bmx^*$, i.e., $\int {(\mu_{\bmx^*} - \tilde{\mu}_{\bmx^*})}^2 dP(\bmx^*)$, where $P$ is a measure for $\bmx^*$
%(e.g., see Lemmas~6 and~7 in~\cite{rudi2015less} for details).

% \textbf{Few Assumptions:}
% We derive our bounds with no restrictive assumptions, while a part of the existing methods requires some assumptions.
% Namely, the SR methods generally use the incoherence property of the sub-sampled regressors, and the RFE methods restrict kernel functions as fit to Fourier features.

\section{Minimizing the Normalized Loss}\label{sec:loss}

In this section, we prove Theorem~\ref{the:loss-linfty}.

To show that $\min_{\bmv}\ell_{K,\bmk,\lambda}(\bmv)$ and $\min_{\tilde{\bmv}}\ell_{K_{SS},\bmk_S,\lambda}(\tilde{\bmv})$ are close, we want to say that $K$ and $K_{SS}$ are close in some sense.
Here, we measure their distance by the cut norm of their corresponding graphons $\calK$ and $\calK_{SS}$ in order to exploit Lemma~\ref{lem:Hayashi}.
In the case of $\bmk$ and $\bmk_S$, we measure their distance by the cut norm of the graphons $\calk 1^\top$ and $\calk_S 1^\top$, where $1\colon[0,1]\to \bbR$ is a function with $1(x)=1$.

As we work on graphons, it is useful to define an analog of~\eqref{eq:normalized-loss} for graphons:
\begin{align*}
  \ell_{\calK,\calk,\lambda}(f) & =  \|\calK f- \calk\|_2^2 + \lambda \langle f, \calK f \rangle.
\end{align*}

% We will show that  by showing that $\min_{\bmv}\ell_{K,\bmk,\lambda}(\bmv) \approx \min_{f}\ell_{\calK,\calk,\lambda}(f)$, $\min_{f}\ell_{\calK,\calk,\lambda}(f) \approx \min_{f}\ell_{\calK_{SS},\calk_S,\lambda}(f)$, and $\min_{f}\ell_{\calK_{SS},\calk_S,\lambda}(f) \approx \min_{\tilde{\bmv}}\ell_{K_{SS},\bmk_S,\lambda}(\tilde{\bmv})$.

We show that the minima of $\ell_{K,\bmk,\lambda}$ and $\ell_{K_{SS},\bmk,\lambda}$ are close if $\calK$ and $\calK_{SS}$ are close in the cut norm up to a measure-preserving bijection and so do $\calk 1^\top$ and $\calk_S 1^\top$.
\begin{lemma}\label{lem:subsampling-linfty}
  If a set $S \subseteq [n]$ satisfies
  \[
    \|\calK - \pi(\calK_{SS})\|_\square \leq \epsilon L \text{ and } \|\calk 1^\top - \pi(\calk_S)1^\top\|_\square \leq \epsilon L
  \]
  for some measure-preserving bijection $\pi\colon[0,1]\to[0,1]$, then we have
  \[
    \min_{\tilde{\bmv} \in \bbR^s} \ell_{K_{SS},\bmk_S,\lambda}(\tilde{\bmv})
    =
    \min_{\bmv \in \bbR^n} \ell_{K,\bmk,\lambda}(\bmv)   \pm O\Bigl(\epsilon L^2 R^2 \Bigr).
  \]
\end{lemma}

\begin{proof}[Proof of Theorem~\ref{the:loss-linfty}]
  By applying Lemma~\ref{lem:Hayashi} on $K$ and $\bmk \bmone^\top$, we obtain
  \[
    \|\calK - \pi(\calK_{SS})\|_\square \leq \epsilon L \text{ and } \|(\calk - \pi(\calk_S))1^\top\|_\square \leq \epsilon L
  \]
  for a measure-preserving bijection $\pi\colon[0,1]\to[0,1]$   with probability at least 0.99.
  Then, the theorem follows by Lemma~\ref{lem:subsampling-linfty}.
\end{proof}

%!TEX root=./main.tex

\section{Prediction}\label{sec:prediction}

In this section, we prove Theorem~\ref{the:prediction-linfty}.
The following lemma is a modification of Lemma~\ref{lem:subsampling-linfty} for relating the solution of $\ell_{K_{SS},\bmk_S,\lambda}$ and that of $\ell_{K,\bmk,\lambda}$ using a given measure-preserving bijection.
\begin{lemma}\label{lem:subsampling-linfty-with-solution}
  If a set $S \subseteq [n]$ satisfies
  \[
    \|\calK - \pi(\calK_{SS})\|_\square \leq \epsilon L \text{ and } \|\calk 1^\top- \pi(\calk_S) 1^\top\|_\square \leq \epsilon L
  \]
  for a measure-preserving bijection $\pi\colon[0,1]\to[0,1]$, then for any $\tilde{\bmv} \in \bbR^s$ with $\|\tilde{\bmv}\|_\infty \leq R$, there exists $\bmv \in \bbR^n$ such that
  \[
    \ell_{K_{SS},\bmk_S,\lambda}(\tilde{\bmv})
    =
    \ell_{K,\bmk,\lambda}(\bmv) \pm O\bigl(\epsilon L^2 R^2 \bigr)
    \quad \text{and} \quad
    \pi(\tilde{\calv}) = \calv.
  \]
\end{lemma}

The following lemma states that, if $\bmv + \bmDelta$ and $\bmv$ have similar normalized losses, then $\Phi(\bmDelta)$ must be small in $\calH$-norm.
\begin{lemma}\label{lem:l2-norm-of-perturbation}
  For any vectors $\bmv$, and $\bmDelta\in \bbR^n$, we have
\begin{align*}
    &\|\Phi(\bmDelta)\|_\calH
    = O\biggl(n\sqrt{\frac{\ell_{K,\bmk,\lambda}(\bmv+\bmDelta) - \ell_{K,\bmk,\lambda}(\bmv)}{\lambda }}\biggr).
\end{align*}
\end{lemma}

\begin{proof}[Proof of Theorem~\ref{the:prediction-linfty}]
  On applying Lemma~\ref{lem:Hayashi} to $K$, $\bmk \bmone^\top$, and $\bmy \bmone^\top$, we have
  \begin{align*}
  \|\calK - \pi(\calK_{SS})\|_\square \leq \epsilon L, \|\calk - \pi(\calk_S)\|_\square \leq \epsilon L,
  \end{align*}
  and
  \begin{align*}
  \|\caly - \pi(\caly_S)\|_\square \leq \epsilon L,
  \end{align*}
  which holds for a some measure-preserving bijection $\pi\colon[0,1]\to[0,1]$   with a probability of at least 0.99.
  In what follows, we assume that this has happened.

  Let $\tilde{\bmv}^* \in \bbR^s$ be the minimizer of $\ell_{K_{SS},\bmk_S,\lambda}$ that is returned by Algorithm~\ref{alg:loss}, and let $\bmv \in \bbR^n$ be the vector given by Lemma~\ref{lem:subsampling-linfty-with-solution} on $\tilde{\bmv}^*$.
  Then, we have
\begin{align*}
    \ell_{K,\bmk,\lambda}(\bmv) &= \ell_{K_{SS},\bmk_S,\lambda}(\tilde{\bmv}^*) + O\Bigl(\epsilon L^2R^2\Bigr) \\
    &= \ell_{K,\bmk,\lambda}(\bmv^*) + O\Bigl(\epsilon L^2R^2\Bigr).
\end{align*}
  This means that $\|\Phi(\bmv-\bmv^*)\|_\calH = O\bigl(\sqrt{\epsilon} LR n\bigr)$ by Lemma~\ref{lem:l2-norm-of-perturbation}.
  Let $\pi$ be the measure-preserving bijection given by Lemma~\ref{lem:subsampling-linfty-with-solution}.
  Then, we have
  \begin{align}
   & \tilde{\mu}_{\bmx^*} = \frac{\langle \tilde{\bmv},\bmy_S\rangle}{s} = \langle \tilde{\calv}, \caly_S \rangle = \langle \pi(\tilde{\calv}), \pi(\caly_S)\rangle = \langle \calv, \pi(\caly_S)\rangle \nonumber \\
  & = \langle \calv, \caly \rangle + \langle \calv, \pi(\caly_S) - \caly \rangle = \frac{\langle \bmv, \bmy \rangle}{n} + \langle \calv, \pi(\caly_S) - \caly \rangle \nonumber \\
  & = \frac{\langle \bmv^*, \bmy \rangle}{n} + \frac{\langle \bmv - \bmv^*, \bmy \rangle}{n} + \langle \calv, \pi(\caly_S) - \caly \rangle \notag\\
  &= \mu_{\bmx^*} + \frac{{\langle \Phi(\bmv - \bmv^*), \bmalpha \rangle}_\calH}{n} + \langle \calv, \pi(\caly_S) - \caly \rangle. \label{eq:prediction-linfty-1}
  \end{align}
  By Cauchy-Schwarz and Lemma~\ref{lem:Hayashi-approximation}, we have
  \begin{align}
  \eqref{eq:prediction-linfty-1} &= \mu_{\bmx^*} \pm \Bigl(\frac{\|\Phi(\bmv^* - \bmv)\|_\calH \|\bmalpha\|_\calH}{n} \notag \\
  & \qquad \qquad + \| \calv\|_\infty \|(\pi(\caly_S) - \caly)1^\top\|_\square \|1\|_\infty \Bigr)
  \notag \\
  &= {\mu}_{\bmx^*} \pm O\left(\sqrt{\epsilon} L^2R \right). \label{eq:prediction-linfty-1.5}
  \end{align}

  Similarly, we have
  \begin{align}
     \tilde{\sigma}^2_{\bmx^*}
    &= k(\bmx^*,\bmx^*) - \frac{\langle \tilde{\bmv},\bmk_S\rangle}{s} \notag
    = k(\bmx^*,\bmx^*) - \frac{\langle \tilde{\calv},\calk_S\rangle}{s} \notag \\
    &= k(\bmx^*,\bmx^*) - \langle \pi(\tilde{\calv}),\pi(\calk_S) \rangle \notag \\
    &= k(\bmx^*,\bmx^*) - \langle \calv,\pi(\calk_S) \rangle \nonumber \\
    & = k(\bmx^*,\bmx^*) - \langle \calv,\calk \rangle - \langle \calv, \pi(\calk_S) - \calk \rangle \notag \\
    &= k(\bmx^*,\bmx^*) - \frac{\langle \bmv,\bmk \rangle}{n}- \langle \calv, \pi(\calk_S) - \calk \rangle \nonumber \\
    & = k(\bmx^*,\bmx^*) - \frac{\langle \bmv^*,\bmk \rangle}{n}- \frac{\langle \bmv - \bmv^*,\bmk \rangle}{n} -  \langle \calv, \pi(\calk_S) - \calk \rangle \notag \\
    &= \sigma^2_{\bmx^*} - \frac{\langle \Phi(\bmv - \bmv^*), \phi_{\bmx^*} \rangle_\calH}{n} - \langle \calv, \pi(\calk_S) - \calk \rangle. \label{eq:prediction-linfty-2}
  \end{align}
  By Cauchy-Schwarz and Lemma~\ref{lem:Hayashi-approximation}, we have
  \begin{align*}
   \eqref{eq:prediction-linfty-2} &= \sigma^2_{\bmx^*} \pm \Bigl(\frac{\|\Phi(\bmv - \bmv^*)\|_\calH \|\phi_{\bmx^*}\|_\calH}{n}\\
   & \qquad \qquad + \|\calv\|_\infty \|(\pi(\calk_S) - \calk)1^\top\|_\square \|1\|_\infty\Bigr)\\
    &= \sigma^2_{\bmx^*} \pm O\left(\sqrt{\epsilon} L^2R \right).
  \end{align*}
\end{proof}

%!TEX root=./main.tex

\section{Application to Hyperparameter Selection}\label{sec:cv}

GPR has several hyperparameters such as $\lambda$ in~\eqref{eq:normalized-loss} and hyperparameters used in kernel functions, e.g., the bandwidth $h>0$ in the Gaussian kernel $k(\bmx,\bmx')=\exp(-h^{-1}\|\bmx-\bmx'\|^2_2)$ and the parameters $h,a,b$ in the polynomial kernel $k(\bmx,\bmx') = {(h^{-1}\langle \bmx, \bmx' \rangle + b)}^a$.
Cross validation (CV)~\cite{geisser1975predictive,stone1974cross} is a popular approach for selecting such hyperparameters, although it is computationally expensive.
In this section, we show that we can circumvent this issue by using our method (Algorithm~\ref{alg:prediction}).

%For each fixed hyperparameters, cross-validation first computes the expected loss. \ynote{what's the expectation over?}
Let $\theta$ be the set of hyperparameters, e.g., $\theta = (\lambda,h)$ for the Gaussian kernel.
We consider a predictor $\hat{f}_{S,\theta}(\bmx^*)$, which is the predictive mean obtained when we run Algorithm~\ref{alg:prediction} on $\bmx^* \in \bbR^p$ with hyperparameters $\theta$ and the index set $S \subseteq [n]$ of size $s \in \bbN$.
Furthermore, let $f^0_\theta(\bmx) := \mu_{\bmx^*}$ be the predictive mean using all the $n$ samples. % in \eqref{def:pred} with $\theta$, such that $f^0(\bmx^*) = \bmk_{\bmx}(K+\lambda I)^{-1}\bmy$.
For any $\theta$, we assume that $f^*$, $\hat{f}_{S,\theta}$, and $f^0_\theta$ are bounded and have finite second moments, i.e., $B:=\max\set{\|f^*\|_{\infty}, \|f^0_\theta\|_{\infty},\|\hat{f}_{S,\theta}\|_{\infty}}$ and $B_{\sigma} :=\max\set{\|f^*\|_{2}^2, \|f^0_\theta\|_{2}^2, \|\hat{f}_{S,\theta}\|_{2}^2}$ are finite.
These assumptions are standard and easy to verify for bounded kernels (Section 4 of~\cite{steinwart2008support} presents detailed discussions).

We want to compute the expected loss of the (original) predictive mean $\mathrm{EL}(\theta) := \Ep_{\bmx}[{(f^*(\bmx) - f^0_\theta(\bmx))}^2]$ for a given $\theta$ and then select the best $\theta$.\footnote{We discuss hyperparameter tuning based on the marginal likelihood in Section~\ref{sec:discussion}.}
To this end, in the CV, we first sample an index set $Q \subseteq [n] \backslash S$ of size $q \leq n - s$ uniformly at random.
%as $D_{p} := \{(\bmx_1,y_1),\ldots,(\bmx_{n_1},y_{m_1})\}$ and $D_{v} := \{(\bmx_{n_1+1},y_{n_1+1}),\ldots,(\bmx_{n},y_{m})\}$ with $s = |D_p|$ and set $q = |D_v|$. \ynote{It seems $m_i$ should be $n_i$ and I don't see why we need $m_i$'s.} \ynote{$s$ is used for a different purpose in previous sections.}
We then define the CV loss as
\begin{align}
  \mathrm{CV}_Q(\hat{f}_{S,\theta}) :=  \frac{1}{q}\sum_{i \in Q} {(y_i - \hat{f}_{S,\theta}(\bmx_i))}^2.
  \label{eq:cv-error}
\end{align}

%The following theorem gives a concentration bound for $\mathrm{CV}_Q(\hat{f}_{S,\theta})$.
%describes how the cross-validation loss approaches the expected loss $\mathrm{EL}(\theta)$.  % , where $f^0$ is an ordinal estimator. \ynote{what's the ordinal estimator?}
%Let $\omega(s)$ be the upper bound on $|\mu_{\bmx^*} - \tilde{\mu}_{\bmx^*}|$ given in Theorem~\ref{the:prediction-linfty}.
%\begin{theorem} \label{thm:cv1}
%  For all $s,q \geq 1$ and all $t > 0$, we have
%    \begin{align*}
%        \Pr \biggl[\Bigl|\mathrm{CV}_Q(\hat{f}_{S,\theta}) - \mathrm{EL}(\theta)\Bigr| \leq t \biggr] \geq 1- \exp\left( -q\min\left\{ \frac{(t -\Omega(s))^2}{8 B_\sigma}, \frac{3(t - \Omega(s) )}{4B} \right\} \right),
%    \end{align*}
%    where $\Omega(s) = \sigma^2 + 4 \sigma B + 2 \sigma \omega(s) + \omega(s)^2 + 4 B\omega(s)$.
%    where $B_\sigma$ is a bound parameter for second moments of $f^*$ and its estimators. \ynote{write more specifically.}
%    In particular, if $q \geq \frac{2 B^2 c}{9 B_\sigma}$ \ynote{what's $c$?}, then
%    \begin{align*}
%        \Pr \biggl[\Bigl|\mathrm{CV}_Q(\theta) - \mathrm{EL}(\theta)\Bigr| \leq t \biggr] \geq 1- \exp\left( -q \frac{(t -\Omega(s))^2}{8 B_\sigma}\right).
%    \end{align*}
%\end{theorem}
%Theorem~\ref{thm:cv1} gives a concentration bound for a probability bound for the error $|\mathrm{CV}_Q(\hat{f}_{S,\theta})- \mathrm{EL}(\theta)|$.
%Although $\Omega(s)$ appears in the probability of the bounds, we can mitigate its effect by increasing $q$.

Now, we evaluate the selection performance of the CV based on Algorithm~\ref{alg:prediction}.
For simplicity, we assume that we have two candidates for the choice of hyperparameters, $\theta_1$ and $\theta_2$.
%, and suppose that $\theta_{1}$ is better than $\theta_2$.% i.e., there exists a constant $\Xi > 0$ and $\mathrm{EL}(\theta_1) + \Xi < \mathrm{EL}(\theta_2)$ holds.
Then, we have the following:
\begin{theorem}\label{thm:cv2}
	Suppose that $\mathrm{EL}(\theta_1) + \Xi < \mathrm{EL}(\theta_2)$ holds for some $\Xi > 0$.
    Let us define $\omega(s)$ as the upper bound on $|\mu_{\bmx^*} - \tilde{\mu}_{\bmx^*}|$ given in Theorem~\ref{the:prediction-linfty}, and a parameter $\tilde{\Xi}(s):=\Xi -3 {\omega(s)}^2 - 2\omega(s) B-\nu^2 (4B+\omega(s))$.
  Then for any $s,q \geq 1$,
   \begin{align*}
      \mathrm{CV}_Q(\hat{f}_{S,\theta_1}) \leq \mathrm{CV}_Q(\hat{f}_{S,\theta_2}),
   \end{align*}
   holds, with probability at least
  \begin{align*}
		1- 4\exp\left( - \frac{q}{2} \left( \frac{\tilde{\Xi}(s)}{B^2} - \frac{9 B_\sigma^2 }{B^4} \right) \right)  -\frac{3\nu^2(4B+2\omega(s))}{q\tilde{\Xi}(s)}.
  \end{align*}
\end{theorem}
Note that $\tilde{\Xi}$ is an increasing function in $\Xi$ and $s$.
Hence, Theorem~\ref{thm:cv2} implies that the probability that the approximated CV succeeds increases as $q$, $s$, and $\Xi$ increase.
%As $q$ increases, the probability \eqref{ineq:prob_cv} increases,
%Larger $\Xi$ and smaller $\omega(s)$ also makes the probability large.

\begin{comment}

\section{Additional Bounds}\label{sec:bounds}

In this section, we provide lemmas for providing some bounds.

Firstly, we consider a bound for $\|\bmalpha^*\|_2$ where $\bmalpha^*$ is the minimizer of the problem in Lemma \ref{lem:minimize}.

%To this end, we introduce some additional notions.
Let $f^0$ be a true function which generates the samples.

\begin{lemma}\label{lem:bound_alpha}
	Let a constant $C_f >0$ satisfy $C_f \geq \|f^0\|_{2}$.
	Then, with probability at least $1-\delta$, there exist a constant $q=q(\delta)>0$ and $s>0$ such that
	\begin{align*}
		\|\bmalpha^*\|_2 \leq \sqrt{m} \left( C_f + q(\delta) n^{-s}\right).
	\end{align*}
\end{lemma}
\begin{proof}
	Let $f$ be a function such as $f(y)=\alpha^*_{i_m(y)}$.
	We have
	\begin{align*}
		\|\bmalpha^*\|_2 &= \sqrt{m}\|f\|_2\\
		&\leq \sqrt{m} \left( \|f^0\|_2 + \|f - f^0\|_2\right)\\
		&\leq \sqrt{m} \left( C_f + \|f - f^0\|_2\right).
	\end{align*}
	About the term $\|f - f^0\|_2$, we apply Theorem 1 in \cite{rudi2015less} and obtain
	\begin{align*}
		\|f - f^0\|_2 \leq q(\delta) n^{-s},
	\end{align*}
	with sufficiently large $n$ and $m$.
	Here, $q(\delta) > 0$ and $s>0$ be a constants defined in \cite{rudi2015less}.
\end{proof}
\end{comment}

%Secondly, we bound an error between the true function and the function by our method.

%!TEX root=./main.tex

\section{Experiments}

\subsection{Approximation Accuracy}\label{sec:exp-approx}

First, we evaluated the performance of subsampling with a constant number of samples that are covered by our theory, that is, the predictive mean/variance~\eqref{eq:mean-and-variance}, the minimum of the normalized loss function~\eqref{eq:normalized-loss}, and the CV error~\eqref{eq:cv-error}. Here, we used five real datasets (libsvm datasets\footnote{\url{https://www.csie.ntu.edu.tw/~cjlin/libsvmtools/datasets}}) whose sample sizes are in thousands such that we could run the exact GPR for comparison. Each data set was standardized beforehand so that $y$ and each feature of $\bmx$ are ranged in $[-1, 1]$.

\begin{figure}[t]
  \centering
  \includegraphics[width=.99\linewidth]{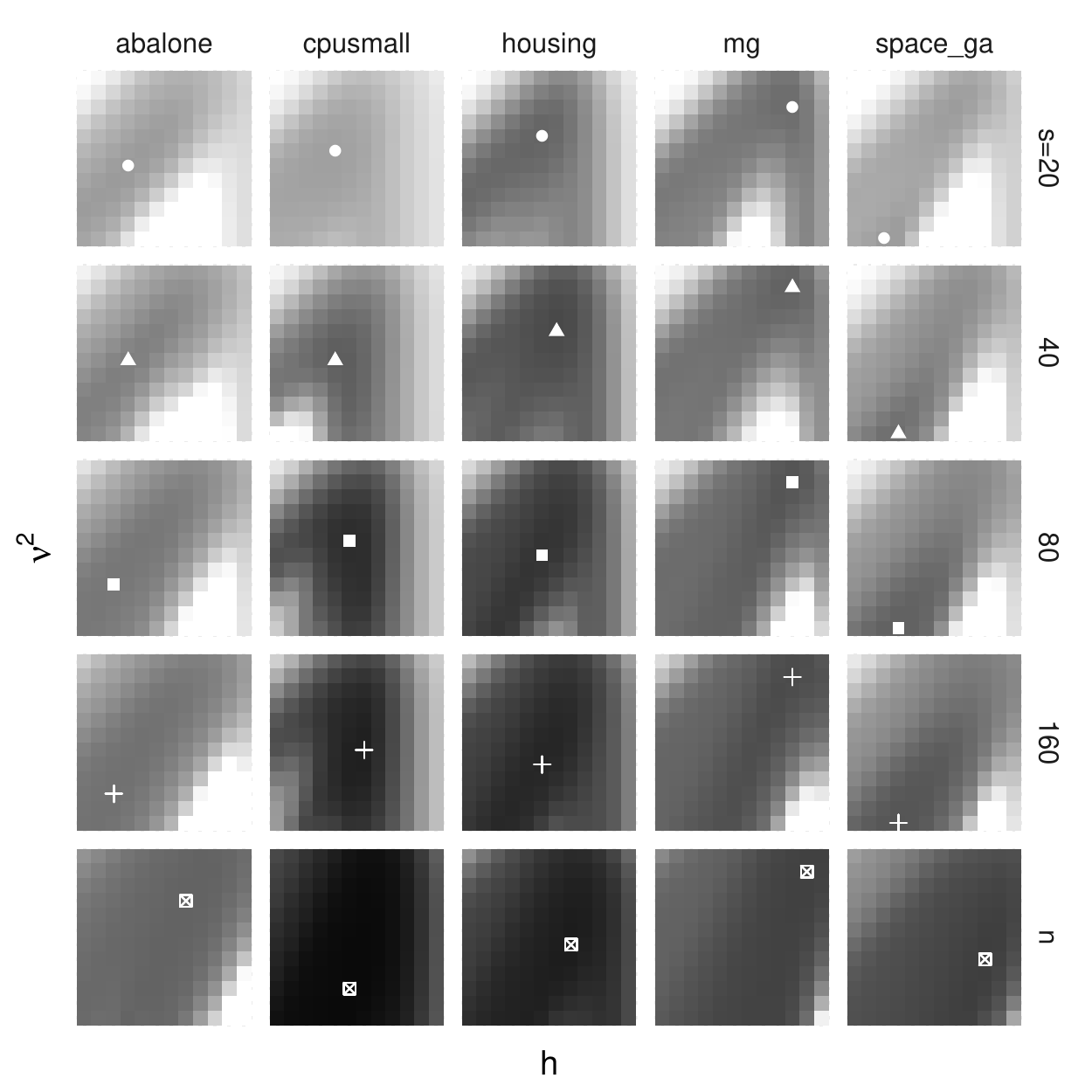}\\
  %\hspace{3em}
  \includegraphics[width=.99\linewidth]{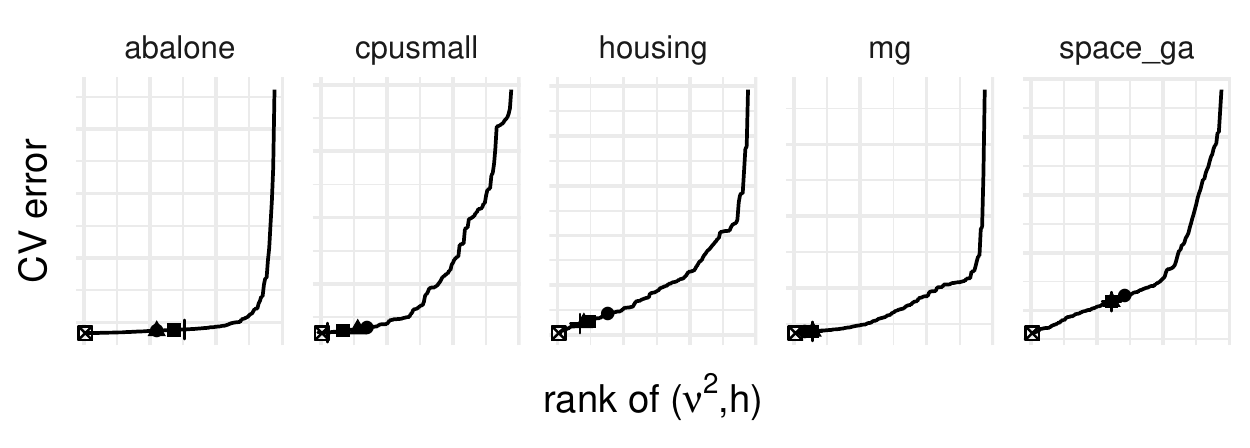}
\caption{Top: Contours of the CV error for the noise variance $\nu^2$ and the kernel bandwidth $h$ in a logarithmic scale ($\nu^2,h^{-1}\in\set{10^{-i/3} \mid i=0,1,\dots,11}$). White dots indicate the selected hyperparameters by CV (i.e., the minima of the contours). Bottom: CV error for each pair $(\nu^2,h)$ in increasing order. The selected $(\nu^2,h)$ with other $s$ are also shown as black dots with the same shape as above. }\label{fig:cverror}
\end{figure}

The upper part of Figure~\ref{fig:cverror} shows the contour of the 10-fold CV error with the Gaussian kernel. In datasets \texttt{housing} and \texttt{mg}, subsampling successfully selected the hyperparameters that were sufficiently close to the ones selected by the full-sample CV\@. In \texttt{abalone} and \texttt{cpusmall}, the selected hyperparameters look far. However, this was because the landscape of the full-sample CV error was flat (the lower part of Figure~\ref{fig:cverror}) and it was difficult to choose the optimal hyperparameters even in the original CV\@. Indeed, this case corresponds to the case that $\Xi$ in Theorem~\ref{thm:cv2} is small, and these empirical results agree with the claim of Theorem~\ref{thm:cv2}: the hyperparameter selection may fail for small $\Xi$.

\begin{figure}[t]
  \centering
  \includegraphics[width=.99\linewidth]{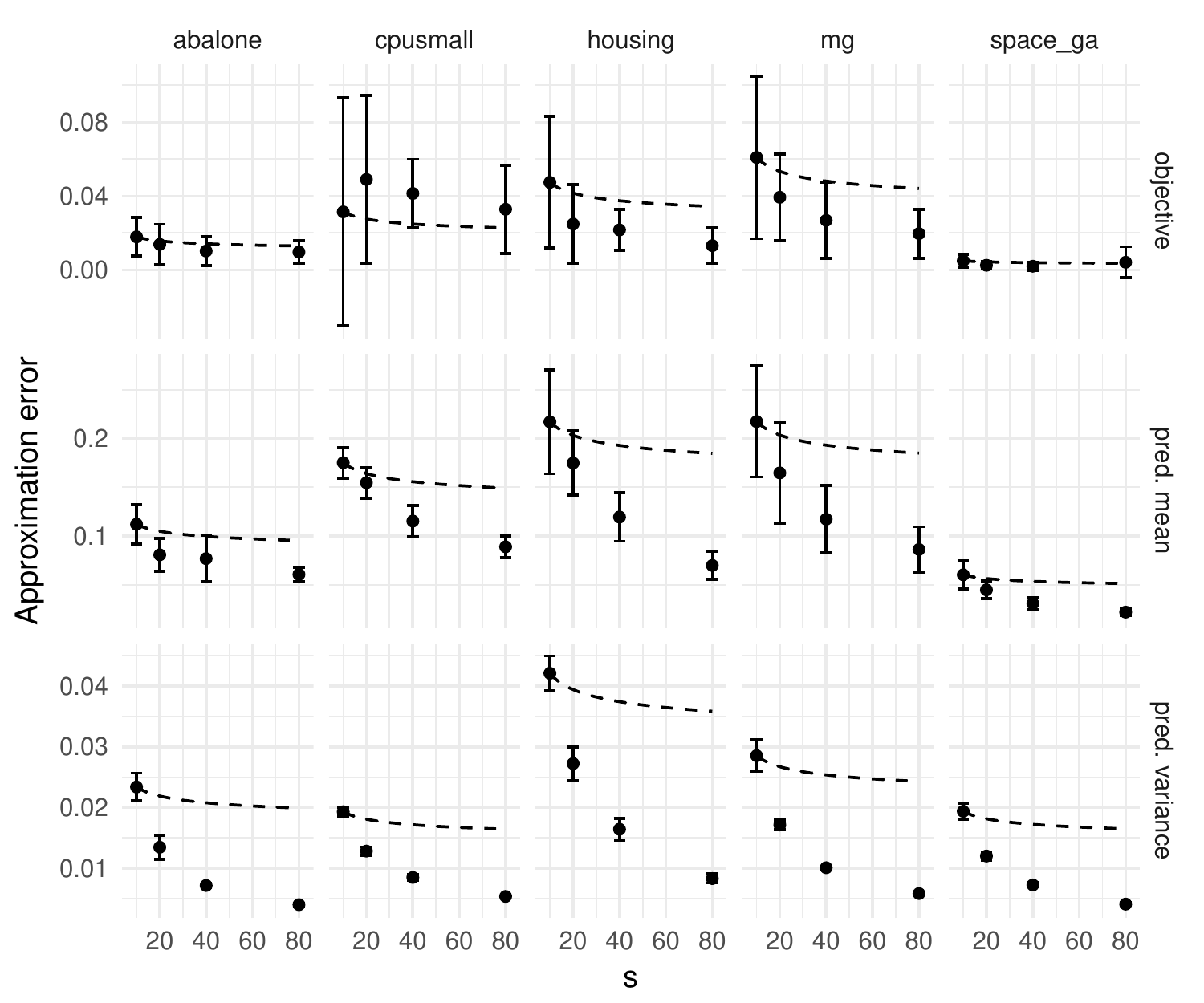}
\caption{Errors of the predictive mean/variance and the normalized loss function with the Gaussian kernel. The hyperparameters were set as the noise variance $\nu^2=0.01$ and the bandwidth $h=10$. The error bars indicate the standard deviation of the results over ten trials with different random seeds. The dashed lines indicate the theoretical bounds (Theorems~\ref{the:loss-linfty} and~\ref{the:prediction-linfty}) where we set unknown linear coefficients of the bounds as they fit the results.}\label{fig:obj}
\end{figure}

Figure~\ref{fig:obj} shows the errors of the predictive mean $|\mu_{\bmx^*} - \tilde{\mu}_{\bmx^*}|$, predictive variance $ |\sigma_{\bmx^*}^2 - \tilde{\sigma}_{\bmx^*}^2|$, and the objective $|\ell_{K,\bmk,\lambda}(\bmv^*) - \ell_{K_{SS},\bmk_S,\lambda}(\tilde{\bmv}^*)|$ with the Gaussian kernel.
We see that the errors, especially of the predictive mean and variance, decrease faster than we expect from the theoretical convergence rate of $O(\log ^{-1/4} s)$ shown in the dashed lines.

\begin{figure}[t]
  \centering
  \includegraphics[width=.99\linewidth]{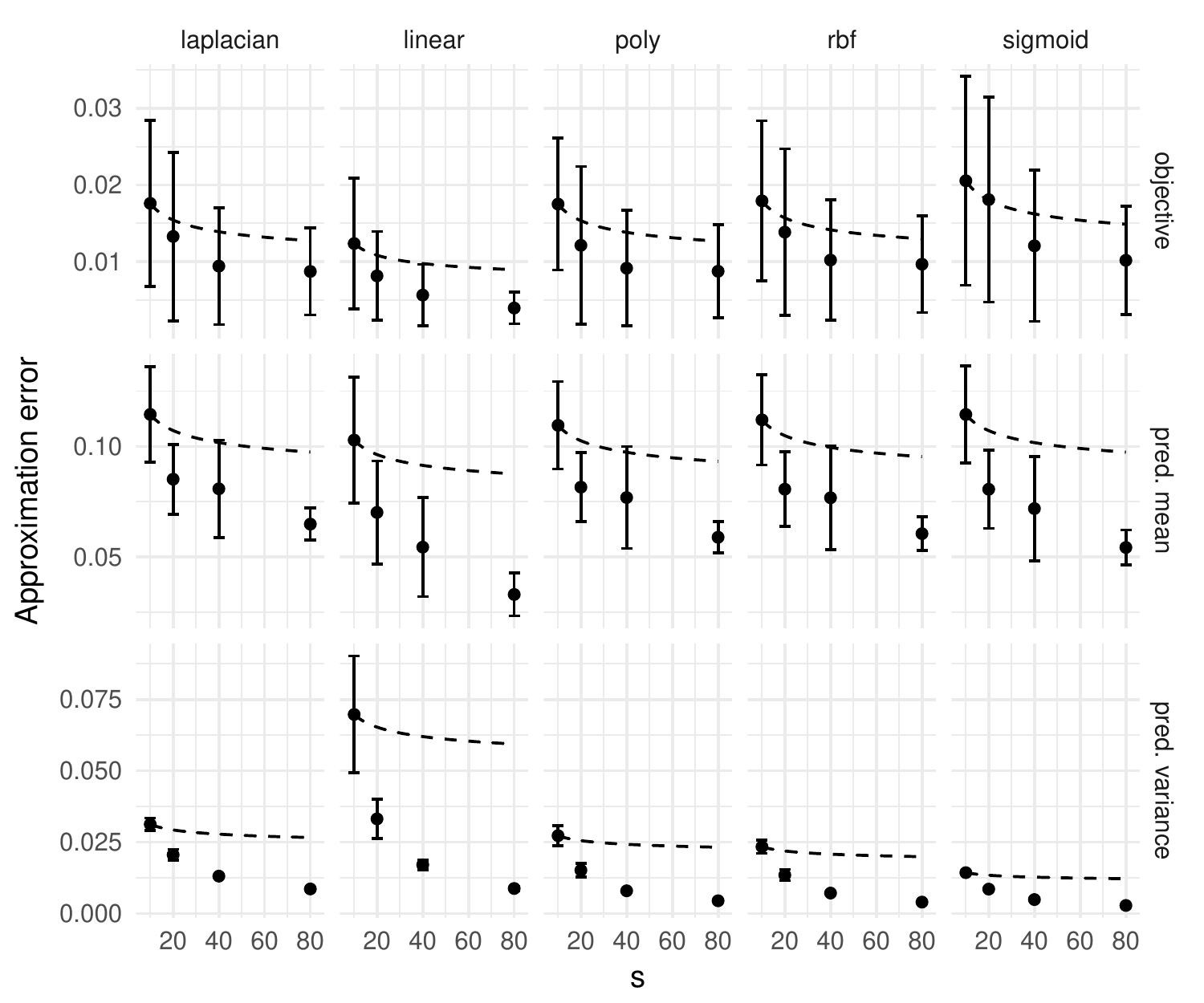}
\caption{Approximation errors with the Laplacian, linear, polynomial, Gaussian (RBF), and sigmoid kernel functions on \texttt{abalone} data set. The hyperparameters were set as the noise variance $\nu^2=0.01$ and the bandwidth $h=10$. Other kernel parameters were fixed as the default values of scikit-learn library.}\label{fig:variouskernel}
\end{figure}

We also investigated how the choice of kernel functions affects the approximation quality.  Figure~\ref{fig:variouskernel} shows a similar behavior as in Figure~\ref{fig:obj} no matter which kernel function is used. We observe that all kernel functions behave very similarly, meaning subsampling works independently of the choice of the kernel function as our theory suggested. Due to the page limitation, we only show the result with a single data set here; see Appendix~\ref{supp:variouskernel} for the complete results.

\subsection{Prediction Accuracy and Runtime}

\begin{figure}[t]
  \centering
  \includegraphics[width=.86\linewidth]{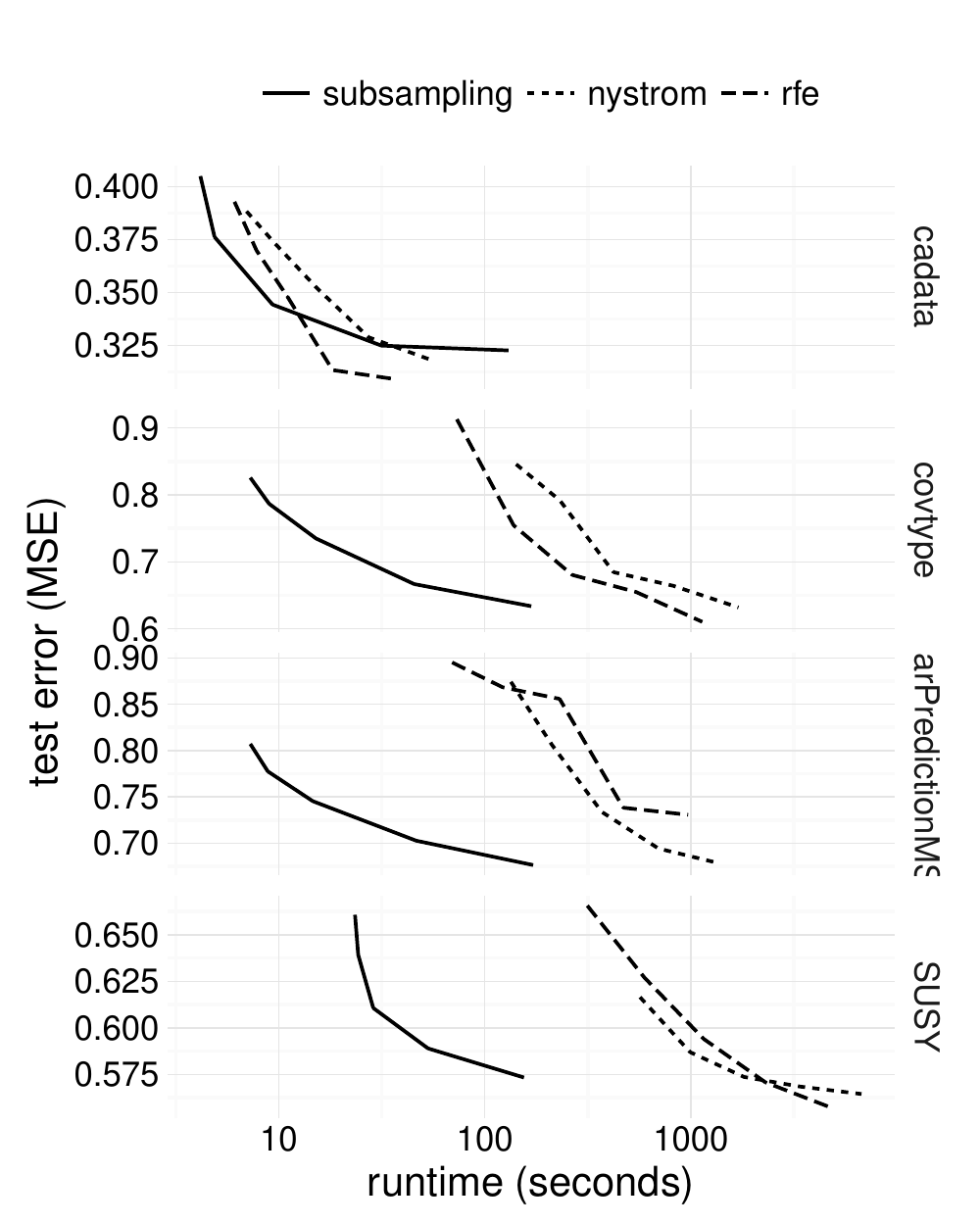}
\caption{Comparison of test error and runtime. We changed $s$ over $\set{10\cdot 2^{i}\mid i\in\set{4,\dots,8}}$ for subsampling, and changed the rank and the dimension of the feature functions over $\set{10\cdot 2^{i}\mid i\in\set{1,\dots,5}}$
for the \nystrom method and RFE, respectively.}\label{fig:tradeoff}
\end{figure}

Next, we compared the prediction performance with the \nystrom method and RFE. Specifically, we are interested in the trade-off between the prediction accuracy on the test data (i.e., the generalization power) and the runtime. To this end, we prepared relatively large-scale datasets: \texttt{cadata} ($n\simeq$ 20K), \texttt{YearPredictionMSD} ($\simeq$ 0.4M), \texttt{covtype} ($\simeq$ 0.6M), and \texttt{SUSY} ($\simeq$ 5M). Note that  the labels of \texttt{covtype} and \texttt{SUSY} were binary but we regarded them as real values. To evaluate the prediction performance, we split each data set into a test set consisting of $1,000$ randomly selected samples and a training set consisting of the rest of the samples.
We selected the hyperparameters by 3-fold CV for each method.
Note that all the methods were implemented in Python and their runtime was recorded on an Amazon EC2 r4.16xlarge instance.

Figure~\ref{fig:tradeoff} depicts the trade-off curves between the test error and the runtime for CV and prediction. Again, subsampling showed convincing results.
In \texttt{cadata}, the smallest data set, subsampling and \nystrom were competitive, and RFE was slightly better than them. However, in the large datasets, the curves of subsampling were consistently located in the left-bottom side, meaning that subsampling significantly extends the Pareto frontier in terms of the trade-off. Note that all of those approximation methods converge to the exact GPR so that, by increasing runtime, they eventually end up at the same error.

%they become distant as $n$ increases.
%This behavior is as expected because the runtime of subsampling is independent of $n$, whereas that of the others depend. In addition, the prediction ability of subsampling is no worse than others.
% results demonstrate that subsampling has an equivalent or better  than others.

%!TEX root=./main.tex

\section{Discussion}\label{sec:discussion}

In this work, we explored the theoretical aspects of random subsampling of GPR\@. Using graphons, we built the error bounds for the predictive distribution and generalization. Although the derived rates are slower than other structure-based approximations, they only require minimum assumptions. The experimental results demonstrated that subsampling achieves a better speed-accuracy trade-off than the \nystrom and RFE methods when the number of samples $n$ is sufficiently large (say, $n>10^5$). Combining the theoretical and empirical results, we conclude that subsampling is worth a try as well as more other complicated approximations.

%The empirical results (Figures~\ref{fig:obj}\ref{fig:variouskernel}\ref{fig:tradeoff})
The empirical results (Figures~\ref{fig:obj}--\ref{fig:tradeoff}) repeatedly indicate that the actual performance of subsampling is far better than theoretically expected. This would be because the derived bounds (Theorem~\ref{the:prediction-linfty} and Corollary~\ref{cor:gen}) are too conservative. Actually, they consider almost worst-case scenarios, such as the truth function is peaky everywhere or drawn subsamples are densely collected in a small input area. Adding some realistic assumptions such as smoothness may help to derive better error bounds.

We have shown that the CV strategy well admits subsampling (Section~\ref{sec:cv}), but we may want to use subsampling to approximate other criteria. The marginal likelihood would be the most popular criterion in the GP community for hyperparameter selection~\cite{rasmussen2004gaussian}. Unfortunately, our analysis is not immediately applicable to approximating it. Let us explain why. The marginal likelihood has the explicit form of $\log\det (K + n\lambda I) + \langle\bmy, {(K + n\lambda I)}^{-1} \bmy\rangle + n/2 \log 2\pi $. The second term has the quadratic form as we have already seen (e.g., Eq.~\ref{eq:the-problem}) and indeed subsampling can approximate it. The difficulty is in the first term, which we have to deal with the determinant of the kernel matrix. Remember that we treat the kernel matrix as the graphon in our analysis. However, the determinant of the graphon is not well-defined, meaning that we cannot compare kernel matrices with different sample size, and therefore, the approximation accuracy remains unknown. Further investigation on the marginal likelihood approximation is one of our future works.

\bibliographystyle{abbrv}
\bibliography{main}
%\knote{The reference (Levi and Yoshida) can violate the anonymity: see Double-blind reviewing at \url{https://nips.cc/Conferences/2018/PaperInformation/AuthorGuidelines}. Changing the author names to Anonymous et al. and putting the pdf as the supplementary material when submission would be better.}

\newpage

\appendix
% \input{l2}
%!TEX root=./main.tex
\onecolumn

\begin{center}
    \LARGE Appendix
\end{center}

\section{Proof of Section~ \ref{sec:algo}}

\subsection{Proof of Corollary~\ref{cor:gen}}

\begin{proof}

Firstly, we derive an error of $\tilde{\mu}_{\bmx^*}$ when the observed data follow the regression model \eqref{model:reg}.
Namely, we show the following equality with probability at least $1-\delta$,
\begin{align*}
    \tilde{\mu}_{\bmx^*} &= \mu_{\bmx^*} +  O\left( n^{-1/2} R \right) \pm O \left( \sqrt{\epsilon }L'^2 R \right),
\end{align*}
with the model.
This equality is an analogous of the inequality \eqref{eq:prediction-linfty-1.5} without the assumption of the regression model \eqref{model:reg}.

We start with \eqref{eq:prediction-linfty-1} and obtain
\begin{align}
       & \tilde{\mu}_{\bmx^*} = \frac{\langle \bmv^*, \bmy \rangle}{n} + \frac{\langle \bmv - \bmv^*, \bmy \rangle}{n} + \langle \calv, \pi(\caly_S) - \caly \rangle .\label{eq:append_gen_1}
\end{align}
By the model \eqref{model:reg}, we have $\bmy=\bmf + \bmxi$ where $\bmf := (f(\bmx_1),...,f(\bmx_n))^\top$ and $\bmxi := (\xi_1,...,\xi_n)^\top$, then we obtain
\begin{align*}
    \langle \bmv - \bmv^*, \bmy \rangle &= \langle \bmv - \bmv^*, \bmf \rangle + \langle \bmv - \bmv^*, \bmxi \rangle  = \langle \Phi(\bmv - \bmv^*), \bmalpha \rangle_\calH + \sum_{i \in [n]} \xi_i (v_i - v_i^*). 
\end{align*}
About the second term $\sum_{i \in [n]} \xi_i (v_i - v_i^*)$, we define $\bar{v}_i := (v_i - v_i^*)$, then we have
\begin{align*}
    \sum_{i \in [n]} \xi_i (v_i - v_i^*) \sim \mathscr{N}\left(0,\nu^2 \|\bar{\bmv}\|^2_2\right),
\end{align*}
since $\xi_i \sim \mathscr{N}(0,\nu^2)$ independently and identically.
Then, we apply the tail bound for Gaussian random variables and obtain
\begin{align*}
    %\left|\sum_{i \in [n]} \xi_i (v_i - v_i^*) \right| \leq \sqrt{2} \nu \left( \sum_{i \in [n]} \bar{v}_i^2 \right)^{1/2} \log^{1/2} (1/\delta),
    \left|\sum_{i \in [n]} \xi_i (v_i - v_i^*) \right| \leq \sqrt{2} \nu \|\bar{\bmv}\|_2  \log^{1/2} (1/\delta),
\end{align*}
with probability at least $1-\delta$ for any $\delta \in (0,1)$.
By definition of $\bmv$, it has the same $\ell_\infty$ norm of $\bmv^*$, meaning $\|\bar{\bmv}\|_\infty\leq 2R$. Since $\|\bmu\|_2 \leq \sqrt{n}\|\bmu\|_\infty$ for any $\bmu\in\mathbb{R}^n$, we have $\|\bar{\bmv}\|_2 \leq \sqrt{4n}R$, and
\begin{align*}
    \left|\sum_{i \in [n]} \xi_i (v_i - v_i^*) \right| \leq \sqrt{8n} \nu R \log^{1/2} (1/\delta).
\end{align*}
Substituting the result into \eqref{eq:append_gen_1}, and the Cauchy-Schwartz inequality with Lemma~\ref{lem:Hayashi-approximation} as \eqref{eq:prediction-linfty-1.5} yields
\begin{align*}
    \tilde{\mu}_{\bmx^*} &=\mu_{\bmx^*} \pm O\left( \frac{\sqrt{8} \nu R \log^{1/2} (1/\delta)}{\sqrt{n}}\right)   \pm O\Bigl(\frac{\|\Phi(\bmv^* - \bmv)\|_\calH \|\bmalpha\|_\calH}{n} + \| \calv\|_\infty \|(\pi(\caly_S) - \caly)1^\top\|_\square \|1\|_\infty \Bigr) \\
    &=\mu_{\bmx^*} \pm O\left(\frac{\sqrt{8} \nu R \log^{1/2} (1/\delta)}{\sqrt{n}} \right)  \pm O\left(\sqrt{\epsilon} L'^2R \right).
\end{align*}
Substituting $\delta = 0.01$, then we obtain
\begin{align*}
    \tilde{\mu}_{\bmx^*} = \mu_{\bmx^*} + O\left( n^{-1/2} R \right) \pm O \left( \sqrt{\epsilon }L'^2 R \right).
\end{align*}
When we substitute $\epsilon = O( \log^{1/2} n)$, the second term $O\left( n^{-1/2} R \right)$ is negligible asymptotically in comparison with $O \left( \sqrt{\epsilon }L^2 R \right)$, hence we can ignore the second term as $n \to \infty$.
\end{proof}

\section{Proofs of Section~\ref{sec:loss}}

\subsection{Proof of Lemma~\ref{lem:subsampling-linfty}}

We say that a function $f \colon [0,1] \to \bbR$ is \emph{$n$-block constant} if $f(x)=f(x')$ holds whenever $i_n(x) = i_n(x')$.
For an $n$-block constant $f$, we can find $\bmv\in \bbR^n$ such that $\ell_{K,\bmk,\lambda}(\bmv) = \ell_{\calK,\calk,\lambda}(f)$:
\begin{lemma}\label{lem:block-constant}
  Let $f\colon [0,1] \to \bbR$ be an $n$-block constant function and let $\bmv \in \bbR^n$ be a vector so that $v_j = f^*(x)$ for $x \in [0,1]$ with $i_n(x) = j$ (Note that $\bmv$ is uniquely determined).
  Then, we have
  \[
    \ell_{K,\bmk,\lambda}(\bmv) = \ell_{\calK,\calk,\lambda}(f).
  \]
\end{lemma}
\begin{proof}
  Note that we have
  \begin{align*}
    \ell_{K,\bmk,\lambda}(\bmv) & = \frac{1}{n^3}\|K\bmv\|_2^2 - \frac{2}{n^2}\langle \bmk, K\bmv\rangle + \frac{1}{n}\|\bmk\|_2^2 + \frac{\lambda}{n^2} \langle \bmv, K\bmv  \rangle, \\
    \ell_{\calK,\calk,\lambda}(\bmv) & = \|\calK f\|_2^2 - 2\langle \calk, \calK f\rangle + \|\calk\|_2^2 + \lambda \langle f, \calK f  \rangle.
  \end{align*}
  We show that each pair of corresponding terms are equal.

  For the first pair of terms, we have
  \begin{align*}
    & \|\calK f\|_2^2
    =
    \int_0^1 {\Bigl(\int_0^1 \calK(x,y) f(y) \rmd y\Bigr)}^2 \rmd x
    =
    \sum_{i \in [n]} \int_{I^n_i}{\Bigl(\sum_{j \in [n]}\int_{I^n_j} \calK(x,y) f(y) \rmd y\Bigr)}^2 \rmd x \\
    & =
    \sum_{i \in [n]} \int_{I^n_i}{\Bigl(\sum_{j \in [n]}\int_{I^n_j} K_{ij} v_j \rmd y\Bigr)}^2 \rmd x
    =
    \sum_{i \in [n]} \int_{I^n_i}{\Bigl(\frac{1}{n}\sum_{j \in [n]} K_{ij} v_j \Bigr)}^2 \rmd x  \\
    & =
    \frac{1}{n^3}\sum_{i \in [n]} {\Bigl(\sum_{j \in [n]} K_{ij} v_j \Bigr)}^2
    = \frac{1}{n^3} \|K\bmv\|_2^2.
  \end{align*}
  For the second pair of terms, we have
  \begin{align*}
    & \langle \calk, \calK f\rangle
    =
    \int_0^1 \calk(x)\Bigl(\int_0^1 \calK(x,y) f(y) \rmd y\Bigr) \rmd x \\
    & =
    \sum_{i \in [n]} \int_{I^n_i} \calk(x) \Bigl(\sum_{j \in [n]}\int_{I^n_j}\calK(x,y) f(y) \rmd y\Bigr) \rmd x
    =
    \sum_{i \in [n]} \int_{I^n_i} y_i \Bigl(\sum_{j \in [n]}\int_{I^n_j}K_{ij} v_j \rmd y\Bigr) \rmd x \\
    & =
    \sum_{i \in [n]} \int_{I^n_i} y_i \Bigl(\frac{1}{n} \sum_{j \in [n]} K_{ij}v_j \Bigr) \rmd x
    =
    \frac{1}{n^2} \sum_{i \in [n]} y_i \Bigl(\sum_{j \in [n]} K_{ij} v_j \Bigr)
    = \frac{1}{n^2}\langle \bmk,K\bmv\rangle.
  \end{align*}
  For the third pair of terms, we have
  \begin{align*}
    \|\calk\|_2^2 = \int_0^1 {\calk(x)}^2 \rmd x = \sum_{i \in [n]}\int_{I_i^n} {\calk(x)}^2 \rmd x
    =  \sum_{i \in [n]}\int_{I_i^n} y_i^2 \rmd x
    = \frac{1}{n}\sum_{i \in [n]} y_i^2
    = \frac{1}{n}\|\bmk\|_2^2.
  \end{align*}

  For the fourth pair of terms, we have
  \begin{align*}
    & \langle f, \calK f\rangle
    =
    \int_0^1 \int_0^1 \calK(x,y) f(x) f(y) \rmd x \rmd y
    =
    \sum_{i \in [n]}\sum_{j \in [n]} \int_{I_i^n}\int_{I_j^m} \calK(x,y) f(x) f(y) \rmd x \rmd y \\
    & =
    \frac{1}{n^2}n\sum_{i \in [n]}\sum_{j \in [n]}K_{ij} v_i v_j
    = \frac{1}{n^2}\langle \bmv, K\bmv \rangle.
  \end{align*}
  Combining these equalities establishes the claim.
\end{proof}

The following lemma states that minimizing $\ell_{K,\bmk,\lambda}$ and $\ell_{\calK,\calk,\lambda}$ are equivalent:
\begin{lemma}\label{lem:equivalence-linfty}
  For any $R \in \bbR_+$, we have
  \begin{align*}
    \min_{\bmv \in \bbR^n:\|\bmv\|_\infty \leq R} \ell_{K,\bmk,\lambda}(\bmv)
    =
    \min_{f\colon [0,1] \to \bbR:\|f\|_\infty \leq R} \ell_{\calK,\calk,\lambda}(f).
  \end{align*}
\end{lemma}
\begin{proof}
  First, we show (RHS) $\leq$ (LHS).
  Let $\bmv^*$ be a minimizer of the LHS and let $f\colon[0,1] \to \bbR$ with $f(x) = v^*_{i_n(x)}$.
  Note that $\|f\|_\infty = \|\bmv^*\|_\infty$.
  As $f$ is $n$-block constant, by Lemma~\ref{lem:block-constant}, we have $\ell_{\calK,\calk,\lambda}(f) = \ell_{K,\bmk,\lambda}(\bmv^*)$.

  Next, we show (LHS) $\leq$ (RHS).
  Let $f^*:[0,1]\to \bbR$ be a minimizer of the RHS, which exists because $\ell_{\calK,\calk,\lambda}$ is convex.
  First, we observe that we can assume $f(x)=f(x')$ for every $x,x'\in [0,1]$ with $i_n(x)=i_n(x')$.
  To see this, note that $\ell_{\calK,\calk,\lambda}$ is convex and is invariant under swapping $f(x)$ and $f(x')$ for any $x,x'\in [0,1]$ with $i_n(x)=i_n(x')$.
  Hence, we can decrease the value of $\ell_{\calK,\calk,\lambda}$ by replacing $f(x)$ and $f(x')$ with their average.
  Moreover, $\|f\|_\infty$ does not increase through this modification.
  This means that there is a minimizer $f^*$ of $\ell_{\calK,\calk,\lambda}$ with the desired property.
  Now as $f^*$ is $n$-block constant, Lemma~\ref{lem:block-constant} gives a vector $\bmv \in \bbR^n$ such that $\ell_{K,\bmk,\lambda}(\bmv) = \ell_{\calK,\calk,\lambda}(f^*)$.
  Also, $\|\bmv\|_\infty = \|f^*\|_\infty$.
\end{proof}

\begin{proof}[Proof of Lemma~\ref{lem:subsampling-linfty}]
  We have
  \begin{align}
    & \min_{\tilde{\bmv} \in \bbR^s} \ell_{K_{SS},\bmk_S,\lambda}(\tilde{\bmv}) = \min_{\tilde{\bmv} \in \bbR^s:\|\tilde{\bmv}\|_\infty \leq R} \ell_{K_{SS},\bmk_S,\lambda}(\tilde{\bmv}) = \min_{\substack{f\colon [0,1] \to \bbR: \\ \|f\|_\infty \leq R}} \ell_{\calK_{SS},\calk_S,\lambda}(f) \tag{By Lemma~\ref{lem:equivalence-linfty}} \nonumber \\
    & =
    \min_{\substack{f\colon [0,1] \to \bbR: \\ \|f\|_\infty \leq R}} \| \calK_{SS} f\|_2^2- 2 \langle \calk_S, \calK_{SS} f \rangle + \|\calk_S\|_2^2 + \lambda \langle f, \calK_{SS} f\rangle \nonumber \\
    & = \min_{\substack{f\colon [0,1] \to \bbR: \\ \|f\|_\infty \leq R}} \| \pi(\calK_{SS}) f\|_2^2- 2 \langle \pi(\calk_S), \pi(\calK_{SS}) f \rangle + \|\pi(\calk_S)\|_2^2 + \lambda \langle f, \pi(\calK_{SS}) f\rangle \nonumber \\
    & =
    \min_{\substack{f\colon [0,1] \to \bbR: \\ \|f\|_\infty \leq R}} \Bigl\| \bigl(\pi(\calK_{SS})-\calK+\calK\bigr) f\Bigr\|_2^2- 2 \Bigl\langle \pi(\calk_S) - \calk + \calk, \bigl(\pi(\calK_{SS}) - \calK + \calK\bigr)f \Bigr\rangle   \nonumber \\
    & \qquad \qquad \qquad + \|\pi(\calk_S) - \calk + \calk\|_2^2 + \lambda  \Bigl\langle f, \bigl(\pi(\calK_{SS}) -\calK+\calK\bigr) f\Bigr\rangle \nonumber \\
    & =
    \min_{\substack{f\colon [0,1] \to \bbR: \\ \|f\|_\infty \leq R}} \|\calK f\|_2^2  + 2\Bigl\langle \bigl(\pi(\calK_{SS})-\calK\bigr) f, \calK f \Bigr\rangle  + \Bigl\| \bigl(\pi(\calK_{SS})-\calK\bigr) f\Bigr\|_2^2   \nonumber \\
    &\qquad \qquad\qquad - 2 \langle \calk, \calK f \rangle - 2 \Bigl\langle \calk, \bigl(\pi(\calK_{SS}) - \calK\bigr)f \Bigr\rangle - 2 \bigl\langle \pi(\calk_S)-\calk, \calK f \bigr\rangle \nonumber \\
    & \qquad \qquad\qquad - 2 \Bigl\langle \pi(\calk_S) - \calk, \bigl(\pi(\calK_{SS}) - \calK\bigr)f \Bigr\rangle +  \|\calk\|_2^2 + 2\bigl\langle \pi(\calk_S)-\calk,\calk\bigr\rangle + \bigl\|\pi(\calk_S)-\calk\bigr\|_2^2  \nonumber \\
    & \qquad \qquad\qquad + \lambda  \langle f, \calK f\rangle + \lambda  \Bigl\langle f, \bigl(\pi(\calK_{SS})- \calK\bigr) f\Bigr\rangle. \label{eq:subsampling-linfty-1}
  \end{align}
  By Lemma~\ref{lem:Hayashi-approximation} and using the fact that $\pi(\calk_S) - \calk = (\pi(\calk_S) - \calk) 1^\top 1$, we have
  \begin{align}
    & \eqref{eq:subsampling-linfty-1} =
    \min_{\substack{f\colon [0,1] \to \bbR: \\ \|f\|_\infty \leq R}}  \|\calK f\|_2^2 - 2 \langle \calk, \calK f \rangle + \|\calk\|_2^2 + \lambda  \langle f, \calK f\rangle  \nonumber \\
    &\qquad \qquad \qquad \pm \Bigl(2 \bigl\|\pi(\calK_{SS})-\calK\bigr\|_\square \|\calK\|_\square \|f\|_\infty^2  +  \bigl\|\pi(\calK_{SS})-\calK\bigr\|_\square^2 \|f\|_\infty^2  \nonumber \\
    & \qquad \qquad \qquad + 2 \bigl\|\pi(\calK_{SS}) - \calK\bigr\|_\square \|\calk\|_\infty \|f\|_\infty + 2 \|\calK\|_\square \bigl\|(\pi(\calk_S)-\calk)1^\top \bigr\|_\square \|1\|_\infty \|f\|_\infty \nonumber \\
    & \qquad \qquad \qquad + 2 \bigl\|\pi(\calK_{SS}) - \calK\bigr\|_\square \bigl\|(\pi(\calk_S)- \calk)1^\top\bigr\|_\square \|1\|_\infty \|f\|_\infty   \nonumber \\
    & \qquad \qquad \qquad + 2\bigl\|(\pi(\calk_S)-\calk)1^\top\bigr\|_\square \|1\|_\infty\|\calk\|_\infty + \bigl\|\pi(\calk_S)-\calk\bigr\|_2^2 + \lambda \bigl\|\pi(\calK_{SS})-\calK\bigr\|_\square\|f\|_\infty^2 \Bigr). \label{eq:subsampling-linfty-2}
  \end{align}
  From the assumption, we have
  \begin{align*}
    & =
    \min_{f\colon [0,1] \to \bbR: \|f\|_\infty \leq R} \|\calK f\|_2^2 - 2 \langle \calk, \calK f \rangle + \|\calk\|_2^2 + \lambda \langle f, \calK f\rangle  \\
    &\qquad \qquad\qquad \pm \Bigl(2 \epsilon L^2 R^2  + \epsilon^2 L^2  R^2 + 2 \epsilon L^2 R + 2 \epsilon L^2 R + 2 \epsilon^2 L^2 R + 2\epsilon L^2 + \epsilon^2 L^2 + \lambda \epsilon L R^2\Bigr) \\
    & = \min_{\bmv \in \bbR^n : \|\bmv\|_\infty \leq R} \frac{1}{n^3} \|K \bmv\|_2^2 - \frac{2}{n^2} \langle \bmk, K \bmv \rangle + \frac{1}{n}\|\bmk\|_2^2 + \frac{\lambda}{n^2} \langle \bmv, \calK \bmv \rangle \pm O\Bigl( \epsilon L^2 R^2\Bigr) \tag{By Lemma~\ref{lem:equivalence-linfty}} \\
    & = \min_{\bmv \in \bbR^n : \|\bmv\|_\infty \leq R} \ell_{K,\bmk,\lambda}(\bmv)  \pm O\Bigl(\epsilon L^2 R^2 \Bigr)
    = \min_{\bmv \in \bbR^n } \ell_{K,\bmk,\lambda}(\bmv)  \pm O\Bigl(\epsilon L^2 R^2 \Bigr)
  \end{align*}
  as desired.
\end{proof}

\section{Proofs of Section~\ref{sec:prediction}}

\subsection{Proof of Lemma~\ref{lem:subsampling-linfty-with-solution}}
\begin{proof}
  Let $\tilde{\calv}\colon [0,1]\to \bbR$ be the function corresponding to $\tilde{\bmv}$, that is, $\tilde{\calv}(x) = \tilde{v}_{i_n(x)}$.
  Then, we have
  \begin{align}
    & \ell_{K_{SS},\bmk_S,\lambda}(\tilde{\bmv})
    =
    \frac{1}{s^3}\|K_{SS}\tilde{\bmv}\|_2^2- \frac{2}{s^2} \langle \bmk_S, K_{SS}\tilde{\bmv}\rangle + \frac{1}{s}\|\bmk_S\|_2^2  + \frac{\lambda}{n^2} \langle \tilde{\bmv}, K_{SS} \tilde{\bmv}\rangle \nonumber \\
    & =
    \| \calK_{SS} \tilde{\calv}\|_2^2- 2 \langle \calk_S, \calK_{SS} \tilde{\calv} \rangle + \|\calk_S\|_2^2 + \lambda \langle \tilde{\calv}, \calK_{SS} \tilde{\calv}\rangle \tag{By Lemma~\ref{lem:block-constant}} \nonumber \\
    & = \| \pi(\calK_{SS}) \pi(\tilde{\calv})\|_2^2- 2 \langle \pi(\calk_S), \pi(\calK_{SS}) \pi(\tilde{\calv}) \rangle + \|\pi(\calk_S)\|_2^2 + \lambda \langle \pi(\tilde{\calv}), \pi(\calK_{SS}) \pi(\tilde{\calv})\rangle \nonumber  \\
    & =
    (\pi(\calK_{SS})-\calK+\calK) \pi(\tilde{\calv})\|_2^2- 2 \langle \pi(\calk_S) - \calk + \calk, (\pi(\calK_{SS}) - \calK + \calK)\pi(\tilde{\calv}) \rangle \nonumber  \\
    & \qquad + \|\pi(\calk_S) - \calk + \calk\|_2^2 + \lambda  \langle \pi(\tilde{\calv}), (\pi(\calK_{SS}) -\calK+\calK) \pi(\tilde{\calv})\rangle \nonumber \\
    & =
    \|\calK \pi(\tilde{\calv})\|_2^2  + 2\langle (\pi(\calK_{SS})-\calK) \pi(\tilde{\calv}), \calK \pi(\tilde{\calv}) \rangle  + \| (\pi(\calK_{SS})-\calK) \pi(\tilde{\calv})\|_2^2 \nonumber \\
    &\qquad - 2 \langle \calk, \calK \pi(\tilde{\calv}) \rangle - 2 \langle \calk, (\pi(\calK_{SS}) - \calK)\pi(\tilde{\calv}) \rangle - 2 \langle \pi(\calk_S)-\calk, \calK \pi(\tilde{\calv}) \rangle - 2 \langle \pi(\calk_S) - \calk, (\pi(\calK_{SS}) - \calK)\pi(\tilde{\calv}) \rangle \nonumber \\
    & \qquad +  \|\calk\|_2^2 + 2\langle \pi(\calk_S)-\calk,\calk\rangle + \|\pi(\calk_S)-\calk\|_2^2 \nonumber  \\
    & \qquad + \lambda  \langle \pi(\tilde{\calv}), \calK \pi(\tilde{\calv})\rangle + \lambda  \langle \pi(\tilde{\calv}), (\pi(\calK_{SS})- \calK) \pi(\tilde{\calv})\rangle. \label{eq:subsampling-linfty-with-solution-1}
  \end{align}
  By Lemma~\ref{lem:Hayashi-approximation} and using the assumption that $\pi(\calk_S) - \calk = (\pi(\calk_S) - \calk) 1^\top 1$, we have
  \begin{align}
    & \eqref{eq:subsampling-linfty-with-solution-1} =
    \|\calK \pi(\tilde{\calv})\|_2^2 - 2 \langle \calk, \calK \pi(\tilde{\calv}) \rangle + \|\calk\|_2^2 + \lambda  \langle \pi(\tilde{\calv}), \calK \pi(\tilde{\calv})\rangle \nonumber \\
    &\qquad \pm \Bigl(2 \|\pi(\calK_{SS})-\calK\|_\square \|\calK\|_\square \|\pi(\tilde{\calv})\|_\infty^2  +  \|\pi(\calK_{SS})-\calK\|_\square^2 \|\pi(\tilde{\calv})\|_\infty^2 \nonumber \\
    & \qquad \qquad + 2 \|\pi(\calK_{SS}) - \calK\|_\square \|\calk\|_\infty \|\pi(\tilde{\calv})\|_\infty + 2 \|\calK\|_\square \|(\pi(\calk_S)-\calk)1^\top\|_\square \|1\|_\infty \|\pi(\tilde{\calv})\|_\infty \nonumber \\
    & \qquad \qquad + 2 \|\pi(\calK_{SS}) - \calK\|_\square \|(\pi(\calk_S)- \calk)1^\top\|_\square \|1\|_\infty\|\pi(\tilde{\calv})\|_\infty + 2\|(\pi(\calk_S)-\calk)1^\top\|_\square \|1\|_\infty\|\calk\|_\infty \nonumber \\
    & \qquad \qquad  + \|\pi(\calk)-\calk\|_2^2 + \lambda \|\pi(\calK_{SS})-\calK\|_2\|\pi(\tilde{\calv})\|_2^2 \Bigr). \label{eq:subsampling-linfty-with-solution-2}
  \end{align}
  Recall that $\pi$ satisfies $i_n(\pi(x))=i_n(\pi(y))$ whenever $i_n(x)=i_n(y)$.
  Then, $\pi(\tilde{\calv})$ is $n$-block constant, and hence we can define a vector $\bmv \in \bbR^n$ corresponding to $\pi(\tilde{\calv})$, that is, $v_i = \pi(\tilde{\calv})(x)$ for any $x \in [0,1]$ with $i_n(x)$.
  Then, we have
  \begin{align*}
    & \eqref{eq:subsampling-linfty-with-solution-2} =
    \|\calK \pi(\tilde{\calv})\|_2^2 - 2 \langle \calk, \calK \pi(\tilde{\calv}) \rangle + \|\calk\|_2^2 + \lambda \langle \pi(\tilde{\calv}), \calK \pi(\tilde{\calv})\rangle  \\
    &\qquad \pm \Bigl(2 \epsilon L^2 R^2  + \epsilon^2 L^2  R^2 + 2 \epsilon L^2 R + 2 \epsilon L^2 R + 2 \epsilon^2 L^2 R + 2\epsilon L^2 + \epsilon^2 L^2 + \lambda \epsilon L R^2\Bigr) \\
    & = \frac{1}{n^3} \|K \bmv\|_2^2 - \frac{2}{n^2} \langle \bmk, K \bmv \rangle + \frac{1}{n}\|\bmk\|_2^2 + \frac{\lambda}{n^2} \langle \bmv, \calK \bmv \rangle \pm O\Bigl( \epsilon L^2 R^2\Bigr) \tag{By Lemma~\ref{lem:equivalence-linfty}} \\
    & = \ell_{K,\bmk,\lambda}(\bmv)  \pm O\Bigl(\epsilon L^2 R^2 \Bigr)
  \end{align*}
  as desired.
\end{proof}

\subsection{Proof of Lemma~\ref{lem:l2-norm-of-perturbation}}
\begin{proof}
  Recall that
  \[
    \ell_{K,\bmk,\lambda}(\bmv) = \frac{1}{n^3}\|K\bmv\|_2^2 - \frac{2}{n^2}\langle \bmk, K\bmv\rangle + \frac{1}{n}\|\bmk\|_2^2 + \frac{\lambda}{n^2} \langle \bmv, K\bmv  \rangle.
  \]
  Then, we have
  \begin{align}
  & \ell_{K,\bmk,\lambda}(\bmv+\bmDelta) - \ell_{K,\bmk,\lambda}(\bmv) \nonumber \\
  & =
  \frac{1}{n^3}\|K(\bmv+\bmDelta)\|_2^2 - \frac{1}{n^3}\|K\bmv\|_2^2 - \frac{2}{n^2}\langle \bmk, K(\bmv+\bmDelta)\rangle + \frac{2}{n^2}\langle \bmk, K\bmv\rangle \nonumber \\
  & \quad  + \frac{\lambda}{n^2} \langle (\bmv + \bmDelta), K(\bmv+\bmDelta)  \rangle - \frac{\lambda}{n^2} \langle \bmv, K\bmv  \rangle \nonumber \\
  & = \frac{1}{n^3}\Bigl(2\langle K\bmv, K\bmDelta\rangle + \|K\bmDelta\|_2^2 \Bigr) - \frac{2}{n^2}\langle \bmk,K\bmDelta\rangle + \frac{\lambda}{n^2}\Bigl(2\langle \bmDelta, K(\bmv+\bmDelta)\rangle + \langle \bmDelta, K\bmDelta\rangle\Bigr) \nonumber \\
  & = \frac{1}{n^3}\Bigl(2\langle K\bmv, K\bmDelta\rangle + \|K\bmDelta\|_2^2 \Bigr) - \frac{2}{n^2}\langle \bmk,K\bmDelta\rangle + \frac{\lambda}{n^2}\Bigl(2\langle \bmv, K\bmDelta\rangle + 3\langle \bmDelta, K\bmDelta\rangle\Bigr).   \label{eq:l2-norm-of-perturbation}
%  & = \frac{1}{n^3}\Bigl(2\langle \Phi^\top \Phi \bmv, \Phi^\top \Phi \bmDelta\rangle + \|\Phi^\top \Phi \bmDelta\|_2^2 \Bigr) - \frac{2}{n^2}\langle \bmk,\Phi^\top \Phi \bmDelta\rangle + \frac{\lambda}{n^2}\Bigl(2\langle \bmv, \Phi^\top \Phi \bmDelta\rangle + 3\langle \bmDelta, \Phi^\top \Phi \bmDelta\rangle\Bigr).
  \end{align}
  Let $\lambda_1 \leq Ln$ be the largest eigenvalue of $K$.
  Let $U \Sigma V^\top$ be the SVD of $\Phi$, where $U\in \bbR^{p \times p}$, $\Sigma = \mathrm{diag}(\sigma_1,\ldots,\sigma_p)$ for $\sigma_1 \geq \cdots \geq \sigma_p$, and $V\colon \calH \to \bbR^p$.
  As $K_{ij} = \langle \phi_{\bmx_i},\phi_{\bmx_j}\rangle_{\calH}$, we have $K = U\Sigma^2U^\top$ and hence $\sigma_1 = \lambda_1^{1/2}$
  By Cauchy-Schwarz, we have
  \begin{align*}
    & \eqref{eq:l2-norm-of-perturbation} \geq -\frac{2}{n^3} \|\Phi(K\bmv)\|_\calH\|\Phi(\bmDelta)\|_\calH - \frac{2}{n^2}\|\Phi(\bmk)\|_\calH\|\Phi(\bmDelta)\|_\calH + \frac{\lambda}{n^2}\Bigl( 3\|\Phi(\bmDelta)\|_\calH^2 -2\|\Phi(\bmv)\|_\calH \|\Phi(\bmDelta)\|_\calH \Bigr) \\
    & \geq - \frac{2}{n^{5/2}} \lambda_{\max}^{3/2}R\|\Phi(\bmDelta)\|_\calH - \frac{2\lambda_{\max}^{1/2} L}{n^{3/2}}\|\Phi(\bmDelta)\|_\calH + \frac{\lambda}{n^2}\Bigl( 3 \|\Phi(\bmDelta)\|_\calH^2 -2 \lambda_{\max}^{1/2} Rn^{1/2}\|\Phi(\bmDelta)\|_\calH \Bigr) \\
    & \geq \frac{3 \lambda }{n^2} \|\Phi(\bmDelta)\|_\calH^2 -  \frac{2(L^{3/2} R + L^{3/2}  + \lambda L^{1/2} R)}{n} \|\Phi(\bmDelta)\|_\calH.
%    & \geq \frac{1}{n}\Bigl( \delta^2 L^2 \|\bmDelta\|_2^2  -2 L^2 R\|\bmDelta\|_2 \Bigr) - \frac{2 L^2}{n}\|\bmDelta\|_2 + \frac{\lambda}{n}\Bigl( 3 \delta L \|\bmDelta\|_2^2 -2 L R\|\bmDelta\|_2 \Bigr)\\
%    & = \frac{1}{n}\Bigl( (\delta^2 L^2+3 \delta L) \|\bmDelta\|_2^2  -(2 L^2 R +  2L^2 + 2 L R)\|\bmDelta\|_2  \Bigr).
  \end{align*}

  Then for
  \[
    a = \frac{3\lambda}{n^2} \quad \text{and} \quad b = \frac{2(L^{3/2}R+L^{3/2}+\lambda L^{1/2}R)}{n},
  \]
  we have
  \begin{align*}
    & \|\Phi(\bmDelta)\|_\calH \leq \frac{b-\sqrt{b^2 - 4a(\ell_{K,\bmk,\lambda}(\bmv+\bmDelta) - \ell_{K,\bmk,\lambda}(\bmv))}}{2a} \leq
    \sqrt{\frac{\ell_{K,\bmk,\lambda}(\bmv+\bmDelta) - \ell_{K,\bmk,\lambda}(\bmv)}{a}}\\
    & \leq n\sqrt{\frac{\ell_{K,\bmk,\lambda}(\bmv+\bmDelta) - \ell_{K,\bmk,\lambda}(\bmv)}{3\lambda}} = O\biggl(n\sqrt{\frac{\ell_{K,\bmk,\lambda}(\bmv+\bmDelta) - \ell_{K,\bmk,\lambda}(\bmv)}{\lambda}}\biggr)
  \end{align*}
  as desired.
\end{proof}
The lemma also holds for $\ell_{K_{SS},\bmk_S,\lambda}$ and $\ell_{\calK_{SS},\calk_S,\lambda}$.

\section{Proofs of Section~\ref{sec:cv}}

\subsection{Proof of Theorem~\ref{thm:cv2}}

\begin{proof}
We evaluate the difference between the cross-validated loss values as
\begin{align*}
	&\mathrm{CV}(\theta_1) - \mathrm{CV}(\theta_2) \\
   &=\frac{1}{q} \sum_{i \in Q} {(y_i-\hat{f}_{S,\theta_1}(\bmx_i))}^2 - {(y_i-\hat{f}_{S,\theta_2}(\bmx_i))}^2 \\
   &=\frac{1}{q} \sum_{i \in Q} {(f^*(\bmx_i)-\hat{f}_{S,\theta_1}(\bmx_i))}^2 - {(f^*(\bmx_i)-\hat{f}_{S,\theta_2}(\bmx_i))}^2 \\
   &\quad \quad -2\epsilon_i(f^*(\bmx_i)-\hat{f}_{S,\theta_1}(\bmx_i)) + 2 \epsilon_i(f^*(\bmx_i)-\hat{f}_{S,\theta_2}(\bmx_i))\\
   &=\frac{1}{q} \sum_{i \in Q} {(f^*(\bmx_i)-f^0_{S,\theta_1}(\bmx_i) - \omega_1(s))}^2 - {(f^*(\bmx_i)-f^0_{S,\theta_2}(\bmx_i) - \omega_2(s))}^2 \\
   &\quad \quad -2\epsilon_i(f^*(\bmx_i)-f^0_{S,\theta_1}(\bmx_i) +  \omega_1(s) - f^*(\bmx_i)-f^0_{S,\theta_2}(\bmx_i) - \omega_2(s))\\
   &=\frac{1}{q} \sum_{i \in Q} {(f^*(\bmx_i)-f^0_{S,\theta_1}(\bmx_i))}^2 - \omega_1(s)(f^*(\bmx_i)-f^0_{S,\theta_1}(\bmx_i)) + {\omega_1(s)}^2\\
   &\quad \quad - {(f^*(\bmx_i)-f^0_{S,\theta_2}(\bmx_i))}^2 + \omega_2(s)(f^*(\bmx_i)-f^0_{S,\theta_2}(\bmx_i)) - {\omega_2(s)}^2 \\
   &\quad \quad -2\epsilon_i(f^*(\bmx_i)-f^0_{S,\theta_1}(\bmx_i) - f^*(\bmx_i) +f^0_{S,\theta_2}(\bmx_i)) -2\epsilon_i  (\omega_1(s) - \omega_2(s)).
\end{align*}
Here for $\ell = 1,2$, by the Bernstein's inequality, we have
\begin{align*}
	&\Pr \left( \left|  \frac{1}{q} \sum_{i \in Q} {(f^*(\bmx_i)-f^0_{S,\theta_\ell}(\bmx_i))}^2 - \mathrm{EL}(\theta_\ell)\right| \leq t_\ell \right) \\
    &\geq 1- 2\exp\left( - \frac{1}{2} \frac{t_\ell^2}{B_\sigma^2/q + 2B^2t_\ell /(3q)}\right) =: 1-p_\ell(t_\ell,q),
\end{align*}
for any $t_\ell > 0$.
Also, the Chebyshev's inequality yields
\begin{align*}
	\Pr\left(\left|\nu^2 -   \frac{1}{q} \sum_{i \in Q} \epsilon_i^2\right|\leq t\right) \geq 1-\frac{\nu^2}{q t}=:1-p_\nu(t,q),
\end{align*}
for all $t>0$.
Then, with probability $1-p_1(t_1,q)-p_2(t_2,q)-p_\nu(t_3,q)$, we obtain
\begin{align*}
	&\mathrm{CV}(\theta_1) - \mathrm{CV}(\theta_2)  \\
   &\leq \mathrm{EL}(\theta_1)+ t_1 - \mathrm{EL}(\theta_2 ) + t_2 \\
   &\quad \quad + \frac{1}{q} \sum_{i \in Q}  - \omega_1(s)(f^*(\bmx_i)-f^0_{S,\theta_1}(\bmx_i)) + {\omega_1(s)}^2+ \omega_2(s)(f^*(\bmx_i)-f^0_{S,\theta_2}(\bmx_i)) - {\omega_2(s)}^2 \\
   &\quad \quad -2\epsilon_i(f^*(\bmx_i)-f^0_{S,\theta_1}(\bmx_i) - f^*(\bmx_i) + f^0_{S,\theta_2}(\bmx_i)) -2\epsilon_i  (\omega_1(s) - \omega_2(s))\\
   &\leq -\Xi + t_1 + t_2+3 {\omega(s)}^2 + 2\omega(s) B+\nu^2 (4B+\omega(s)) + t_3(4B+2\omega(s)),  \\
\end{align*}
by applying the Cauchy-Schwarz inequality and $\omega(s) = \omega_1(s) \vee \omega_2(s)$.

Then, we can state that
\begin{align*}
	 \mathrm{CV}(\theta_1) \leq \mathrm{CV}(\theta_2),
\end{align*}
when the following holds;
\begin{align*}
	 t_1 + t_2 + t_3(4B+2\omega)\leq \Xi -3 {\omega(s)}^2 - 2\omega(s) B-\nu^2 (4B+\omega)=:\tilde{\Xi}(s).
\end{align*}
We set $t_1 = t_2 = t_3(4B+2\omega)=\tilde{\Xi}/3$ and substitute them, then we have
\begin{align*}
	 &1 - p_1(t_1,q) -  p_2(t_2,q)-p_\nu(t_3,q)\\
    &= 1- 4\exp\left( - \frac{1}{2} \frac{t^2}{B_\sigma^2/q + B^2t /(3q)}\right) -\frac{3\nu^2(4B+2\omega(s))}{q\tilde{\Xi}(s)}\\
    & \geq 1- 4\exp\left( - \frac{1}{2} \left( \frac{3qt}{B^2} - \frac{B_\sigma^2 9 q}{B^4} \right) \right) -\frac{3\nu^2(4B+2\omega(s))}{q\tilde{\Xi}(s)}\\
    & = 1- 4\exp\left( - \frac{q}{2} \left( \frac{1}{B^2}\tilde{\Xi}(s) - \frac{9 B_\sigma^2 }{B^4} \right) \right)  -\frac{3\nu^2(4B+2\omega(s))}{q\tilde{\Xi}(s)}.
\end{align*}
Then, we obtain the result.
\end{proof}

\section{Approximation Accuracy with Various Kernels with Other Data Sets}\label{supp:variouskernel}

Figures~\ref{fig:abalone}--\ref{fig:space_ga} show the approximation errors with various kernel functions as shown in Section~\ref{sec:exp-approx}, with different datasets.

\begin{figure}[t]
  \centering
  \includegraphics[width=.5\linewidth]{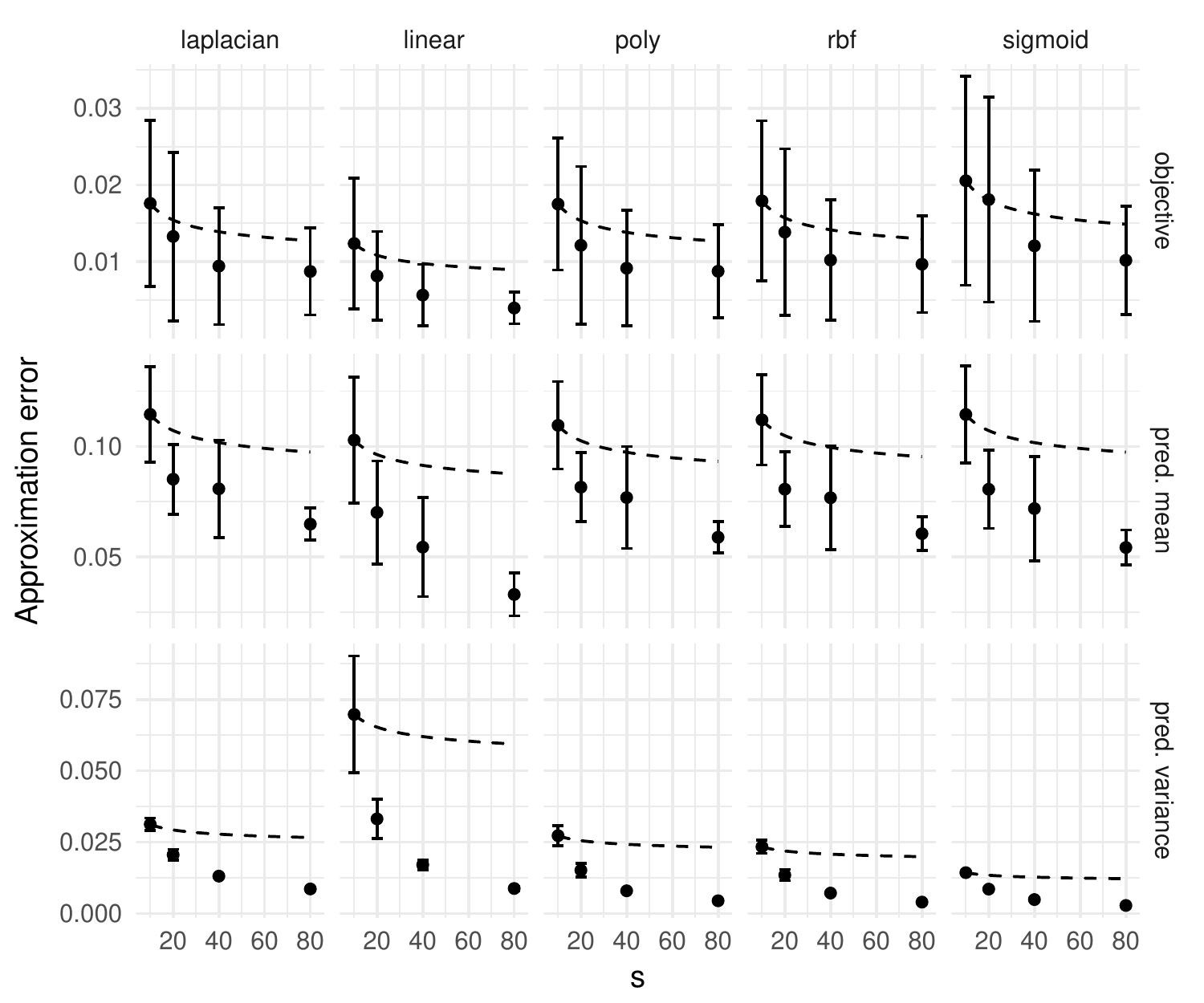}
\caption{Approximation errors on \texttt{abalone} data set. The setting is the same as Section~\ref{sec:exp-approx}.}\label{fig:abalone}
\end{figure}

\begin{figure}[t]
  \centering
  \includegraphics[width=.5\linewidth]{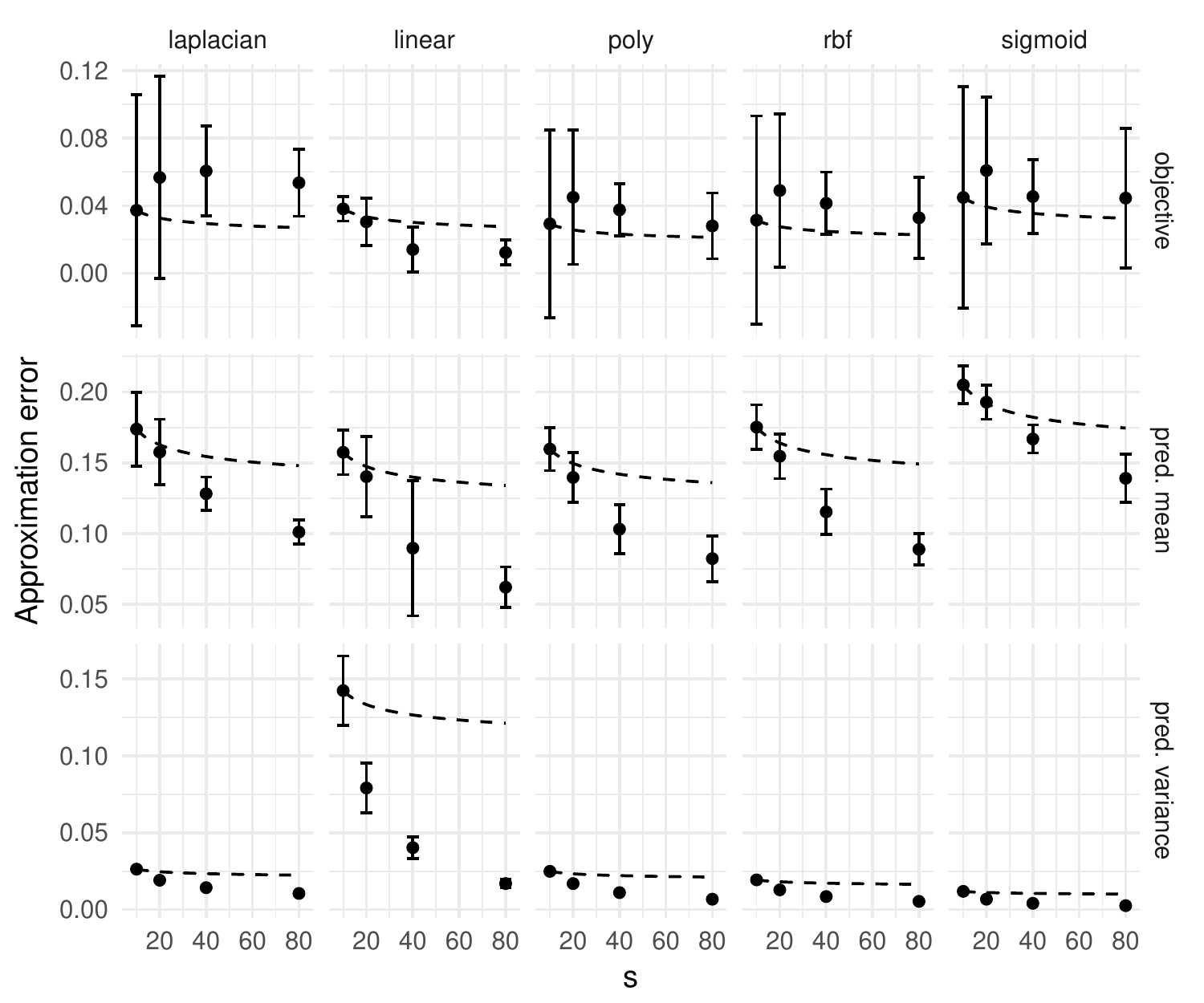}
\caption{Approximation errors on \texttt{cpusmall} data set.}
\end{figure}

\begin{figure}[t]
  \centering
  \includegraphics[width=.5\linewidth]{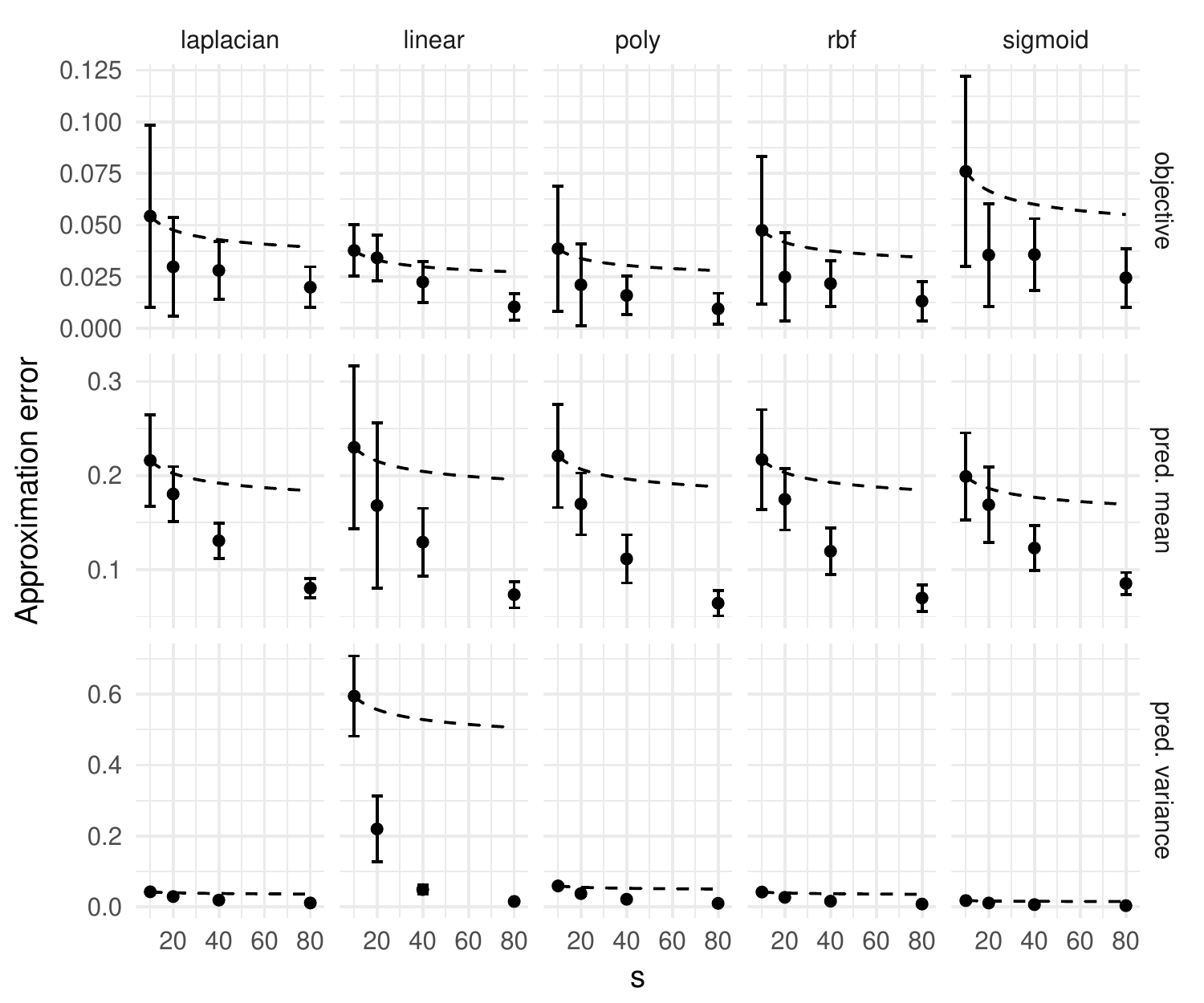}
\caption{Approximation errors on \texttt{housing} data set.}
\end{figure}

\begin{figure}[t]
  \centering
  \includegraphics[width=.5\linewidth]{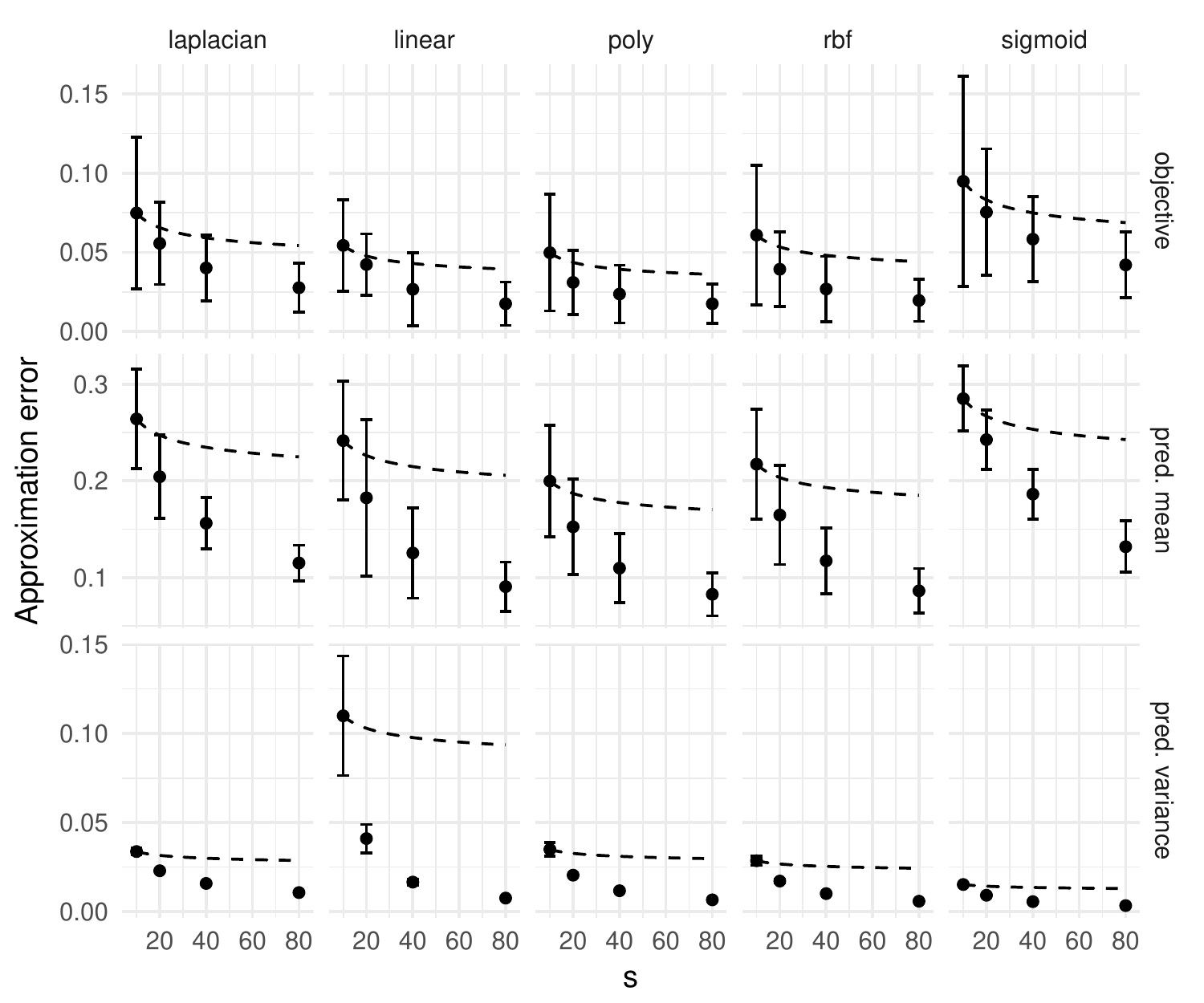}
\caption{Approximation errors on \texttt{mg} data set.}
\end{figure}

\begin{figure}[t]
  \centering
  \includegraphics[width=.5\linewidth]{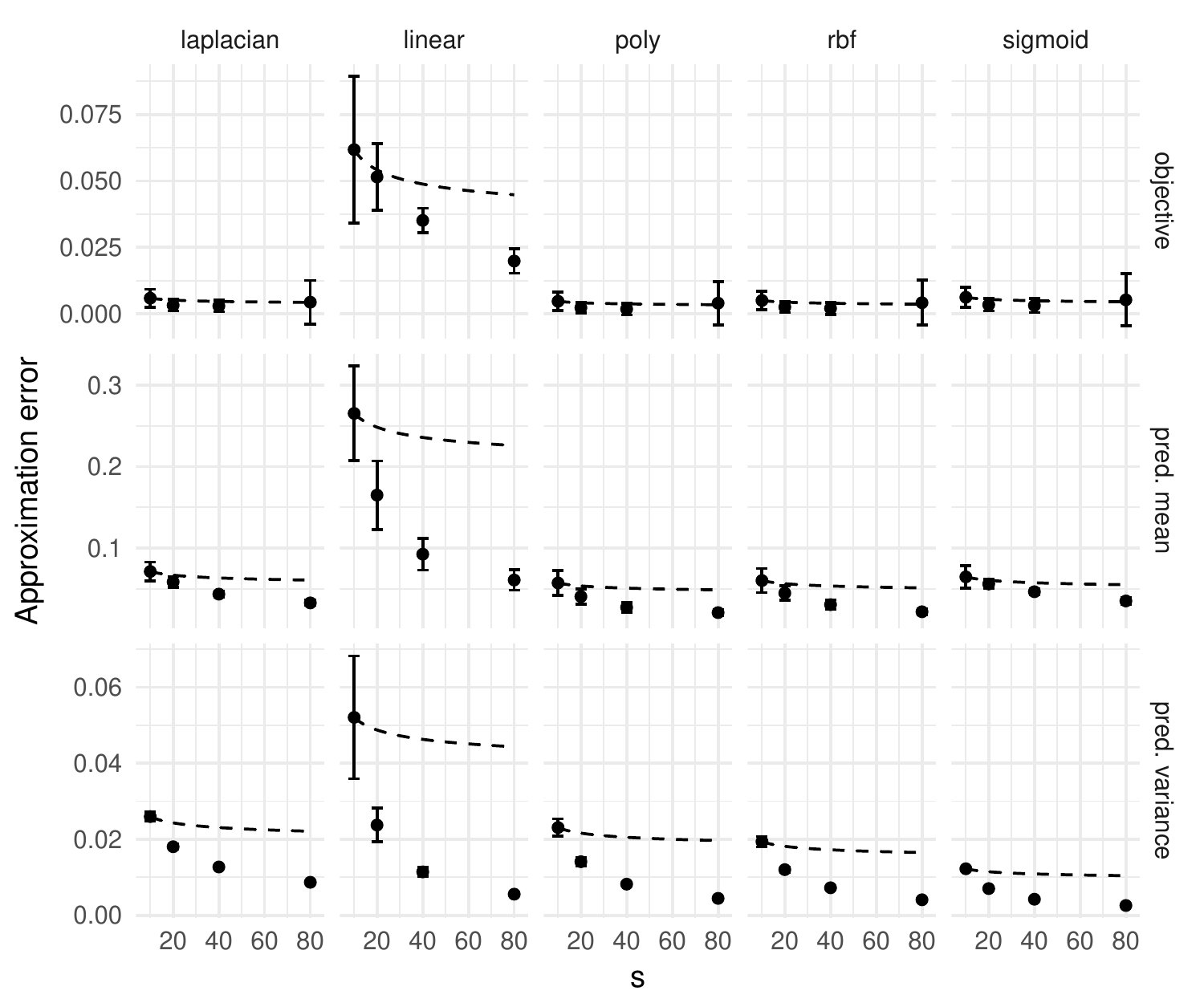}
\caption{Approximation errors on \texttt{space\_ga} data set.}\label{fig:space_ga}
\end{figure}

\end{document}